\documentclass[twoside]{article}

\usepackage[accepted]{arxivpaper2023}

\usepackage[round]{natbib}

\usepackage[utf8]{inputenc} 
\usepackage[T1]{fontenc}    
\usepackage{hyperref}       
\usepackage{url}            
\usepackage{booktabs}       
\usepackage{amsfonts}       
\usepackage{nicefrac}       
\usepackage{microtype}      
\usepackage{algorithm}
\usepackage[noend]{algorithmic}

\usepackage{graphicx}
\graphicspath{{./figures/}}
\usepackage{amsmath,amssymb}
\usepackage{amsthm}
\usepackage{dsfont}

\newcommand{\MMD}{\ensuremath{\mathrm{MMD}}}
\newcommand{\COS}{\ensuremath{\mathrm{COS}}}
\newcommand{\R}{\ensuremath{{\mathbb R}}}

\newtheorem{theorem}{Theorem}
\newtheorem{proposition}{Proposition}
\newtheorem{lemma}[theorem]{Lemma}
\theoremstyle{definition}
\newtheorem{remark}{Remark}

\usepackage[pdftex,dvipsnames]{xcolor}  
\usepackage{wrapfig}
\usepackage{subfigure}
\usepackage{colortbl}
\usepackage{color}
\usepackage{xcolor}
\usepackage{multirow}

\newcommand*\rot{\rotatebox{90}}
\definecolor{verylightgray}{gray}{.75}
\definecolor{veryverylightgray}{gray}{.85}

\usepackage{xargs}                      

\usepackage[colorinlistoftodos,prependcaption,textsize=tiny,disable]{todonotes}

\begin{document}

\twocolumn[

\arxivpapertitle{Validation of Composite Systems by Discrepancy Propagation}

\arxivpaperauthor{ David Reeb \And Kanil Patel \And Karim Barsim \And Martin Schiegg \And Sebastian Gerwinn }

\vspace{-0.05cm}

\begin{center}
	\texttt{\{david.reeb, kanil.patel, karim.barsim, martin.schiegg, sebastian.gerwinn\}\\@de.bosch.com}
\end{center}

\runningauthor{~}
\runningtitle{~}

\vspace{0.05cm}

\arxivpaperaddress{Robert Bosch GmbH, Bosch Center for Artificial Intelligence, 71272 Renningen, Germany} ]

\begin{abstract} 
Assessing the validity of a real-world system with respect to given quality criteria is a common yet costly task in industrial applications due to the vast number of required real-world tests. 
Validating such systems by means of simulation offers a promising and less expensive alternative, but requires an assessment of the simulation accuracy and therefore end-to-end measurements. 
Additionally, covariate shifts between simulations and actual usage can cause difficulties for estimating the reliability of such systems. 
In this work, we present a validation method that propagates bounds on distributional discrepancy measures through a composite system, thereby allowing us to derive an upper bound on the failure probability of the real system from potentially inaccurate simulations. 
Each propagation step entails an optimization problem, where -- for measures such as maximum mean discrepancy (MMD) -- we develop tight convex relaxations based on semidefinite programs. 
We demonstrate that our propagation method yields valid and useful bounds for composite systems exhibiting a variety of realistic effects. 
In particular, we show that the proposed method can successfully account for data shifts within the experimental design as well as model inaccuracies within the simulation.
\end{abstract}

\section{INTRODUCTION}\label{sec:introduction}
Industrial products cannot be released without a priori ensuring their validity, i.e.\ the product must be validated to work according to its specifications with high probability.
Such validation is essential for safety-critical systems (e.g.\ autonomous cars, airplanes, medical machines) or systems with legal requirements (e.g.\ limits on output emissions or power consumption of new vehicle types), see e.g.\ \citep{kalra2016driving,koopman2016challenges,belcastro2003validation}.
When relying on real-world testing alone to validate system-wide requirements, one must perform enough test runs to guarantee an acceptable failure rate, e.g.\ at least $\sim10^6$ runs for a guarantee below $10^{-6}$. This is costly not only in terms of money but also in terms of time-to-release, especially when a failed system test necessitates further design iterations.

System validation is particularly difficult for complex systems which typically consist of multiple components, often developed and tested by different teams under varying operating conditions. 
For example, an advanced driver-assistance system is built from several sensors and controllers, which come from different suppliers but together must guarantee to keep the vehicle safely on the lane. 
Similarly, the powertrain system of a vehicle consists of the engine or battery, a controller and various catalysts or auxiliary components, but is legally required to produce low output emissions of various gases or energy consumption per distance as a whole. 
In both these examples, the validation of the system can also be viewed as the validation of its control component, when the other subsystems are considered fixed. 
To reduce the costs of real-world testing including system assembly and release delays, one can employ \emph{simulations} of the composite system by combining models of the components, to perform {\emph{virtual validation}} of the system \citep{wong2020testing}.

\begin{figure*}[t]
	\centering
	\def\svgscale{0.4}
	\begin{tiny}
		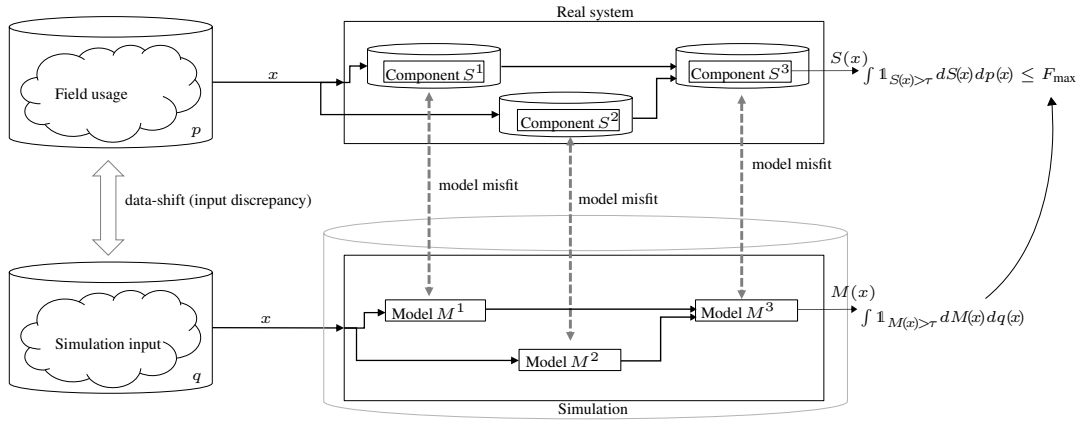 
	\end{tiny}
	\caption{Illustration of our validation task: A real, composite system of interest (top) is modeled with corresponding simulation models (bottom). Measurements of the real system are available only for the individual components, while end-to-end simulation data can be generated from the models. The task of the virtual validation method is to estimate the real system performance $S$ based on the simulations $M$, incorporating simulation model misfits w.r.t.\ the real-world components as well as any data-shift between the simulation input distribution and the field usage to be expected in the real system.
		\label{fig:virtual_system_analysis}}
\end{figure*}

However, it is difficult to assess how much such a composite virtual validation can be trusted, because the component models may be inaccurate w.r.t.\ the real-world components ({\emph{simulation model misfits}}) or the simulation inputs may differ from the distribution of real-world inputs ({\emph{data-shift}}). 
Incorporating these inaccuracies within the virtual validation analysis is particularly important for reliability analyses \citep{bect2012sequential,dubourg2013metamodel,wang2016gaussian} in industrial applications with safety or legal relevance as those described above, where falsely judging a system to be reliable is much more expensive than false negatives. 
For this reason, we desire -- if not an accurate estimate -- then 
at least an upper bound on its true failure probability. 
Existing validation methods are especially lacking the composite (multi-component) aspect, where measurement data are available only for each individual component (Sec.\ \ref{sec:related_work}).

To state the problem mathematically, the goal of this work is to estimate an upper bound $F_\text{max}$  on the failure probability $\mathrm{Pr}\big[S(x)>\tau\big]$ of a system over real-world inputs $x\sim p(x)$:
\begin{align}
F_\text{max} \geq \underset{x,S}{\mathrm{Pr}}\big[S(x)>\tau\big] &= \int \mathds{1}_{S(x)>\tau} dS(x)dp(x), 
\label{eq:failure_prob}
\end{align}
where $S(x)$ measures the system performance upon input $x$, and $\tau$ is a critical performance threshold indicating a \emph{system failure}.
In the virtual validation setup, we assume that no end-to-end measurements from the full composite system $S$ are available, and thus the upper bound $F_\text{max}$ is to be estimated from the simulation $M$ composed of models $M^1,M^2,\ldots$, which are assumed to be given. 
This estimate must take into account \textit{model misfits} and \emph{data-shift} in the simulation input distribution (see Fig.\ \ref{fig:virtual_system_analysis}).
To assess the model misfits, we assume validation measurements from the individual components $S^1,S^2,\ldots$ to be given, e.g.\ from component-wise development (for details see Sec.\ \ref{sec:setup}).

In this paper, we develop a method to estimate $F_\text{max}$ from simulation runs by propagating bounds on distributional distances between simulation models and real-world components through the composite system.
This propagation method incorporates model misfits and data-shifts in a pessimistic fashion by iteratively maximizing for the worst-case output distribution that is consistent with previously computed constraints on the input.
Importantly, our method requires models and validation data from the individual components only, not from the full system.

Our main contributions can be summarized as follows:
\begin{enumerate}
	\item We propose a novel, distribution-free bound on the distance between simulation-based and real-world distributions, 
	without the need to have end-to-end measurements from the real world (Sec.\ \ref{subsec:validation-method}).
	\item We justify the method theoretically (Prop.\ \ref{prop:convergence}) and show its practicality in reliability benchmarks (Sec.\ \ref{sec:reli-benchm-eval}).
	\item We demonstrate that -- in contrast to alternative methods -- the proposed method can account for data-shifts as well as model inaccuracies (Sec.\ \ref{sec:experiments}).
\end{enumerate}

\section{RELATED WORK}\label{sec:related_work}
Estimating the failure probability of a system is a core task in reliability engineering.
In the reliability literature, one focus is on making this estimation more efficient compared to naive Monte Carlo sampling by reducing the variance on the estimator of the failure probability.
Such classical methods include importance sampling \citep{rubinstein2004cross}, 
subset sampling \citep{au2001estimation}, 
line sampling \citep{pradlwarter2007application}, and
first-order \citep{hohenbichler1987new,du2012first,zhang2015first} or second-order \citep{kiureghian1991efficient,lee2012novel} Taylor expansions.
While being more efficient, they still require a large number of end-to-end function evaluations and {cannot incorporate more detailed simulations}.

Another line of research investigates how to reduce real-world function evaluations through virtualization of this performance estimation task \citep{xu2021machine}.
The failure probability is estimated based on a surrogate model and hence cannot account for mismatches between the system and its surrogate.
\citet{dubourg2013metamodel} proposed a hybrid approach, where the proposal distribution of the importance sampling depends on the learned surrogate model. 
While this approach accounts, to some extent, for model mismatches, the proposal distribution might still be biased by a poor surrogate model.
In summary, none of the approaches that are based on surrogate models provide a reliable bound on the true failure probability.
Furthermore, all these approaches require end-to-end measurements from the real system, 
ignoring the composite structure of the system.

In practice, however, the system output $S(\cdot)$ in Eq.\ \eqref{eq:failure_prob} refers to a complex system that often has a \emph{composite structure}.
That is, global inputs $x$ propagate through an arrangement, oftentimes termed a \emph{function network}, of subsystems or components, see Fig.\ \ref{fig:virtual_system_analysis}.
Exploiting such a structure is expected to have a notable impact on the target task, be it experimental design \citep{marque2019efficient},
calibration and optimization \citep{astudillo2019bayesian,astudillo2021bayesian,kusakawa2021bayesian,xiao2022projection}, 
uncertainty quantification \citep{sanson2019systems}, 
or system validation as presented here.

In the context of Bayesian Optimization (BO), for example, \citet{astudillo2021bayesian} 
construct a surrogate system of Gaussian Processes (GP) that mirrors the compositional structure of the system.
Similarly, \citet{sanson2019systems} discuss similarities of such structured surrogate models to Deep GPs \citep{damianou2013deep}, and extend this framework to local evaluations of constituent components.
However, learning (probabilistic) models of inaccuracies \citep{sanson2019systems,riedmaier2021unified} introduces further modeling assumptions and cannot account for data-shifts.
Instead, we aim at model-free worst-case statements.

\citet{marque2019efficient} showed that a composite function can be efficiently modeled from local evaluations of constituent components in a sequential design approach.
\citet{friedman2021adaptive} extend this framework to cyclic structures of composite systems for adaptive experimental design.
They derive bounds on the simulation error in composite systems, although assuming knowledge of Lipschitz constants as well as uniformly bounded component-wise errors.

Stitching different datasets covering the different parts of a larger mechanism without loosing the causal relation was analyzed by \citet{chau2021bayesimp} and corresponding models were constructed, but the quality with which statements about the real mechanism can be made was not analyzed.

Bounding the test error of models under input-datashift was analyzed empirically in \citet{jiang2021assessing} by investigating the disagreement between different models. Although they find a correlation between disagreement and test error, the authors do not provide a rigorous bound on the test error (Sec.\ \ref{sec:uncertainty-wrapper-method}) and also cannot incorporate an existing simulation model into the analysis.

\section{METHOD}\label{sec:method}

\subsection{Setup: Composite System Validation}\label{sec:setup}
We consider a \emph{(real) system} or \emph{system under test} $S$ that is composed of subsystems $S^c$ ($c=1,2,\ldots,C$), over which we have only limited information. 
The validation task is to determine whether $S$ conforms to a given specification, such as whether the system output $y=S(x)$ stays below a given threshold $\tau$ for typical inputs $x$ -- or whether the system's \emph{probability of failure}, defined as violating the threshold, is sufficiently low, see Eq.\ (\ref{eq:failure_prob}).
Our approach to this task is built on a \emph{model} $M$ (typically a simulation, with no analytic form) of $S$ that is similarly composed of corresponding sub-models $M^c$. 
The main challenge in assessing the system's failure probability lies in determining how closely $M$ approximates $S$, in the case 
where the system data originate from disparate component measurements, which cannot be combined to consistent end-to-end data.

\textbf{Components and signals.} Mathematically, each component of $S$ -- and similarly for $M$ -- is a (potentially stochastic) map $S^c$, which upon input of a signal $x^c$ produces an output signal (sample) $y^c\sim S^c(\cdot|x^c)$ according to the conditional distribution $S^c$.
The stochasticity allows for aleatoric system behavior or unmodeled influences.
We consider the case where all signals are tuples $x^c=(x^c_1,\ldots,x^c_{d^c_{in}})$, such as real vectors.
The allowed ``compositions'' of the subsystems $S^c$ must be such that upon input of any signal (stimulus) $x$, an output sample $y\sim S(\cdot|x)$ can be produced by iterating through the components $S^c$ in order $c=1,2,\ldots,C$. 
More precisely, we assume that the input signal $x^c$ into $S^c$ is a concatenation of some entries $x|_{0\to c}$ of the overall input tuple $x$ and entries $y^{c'}|_{c'\to c}$ of some \emph{preceding} outputs $y^{c'}$ (with $c'=1,\ldots,c-1$); 
thus, $S^c$ is ready to be queried right after $S^{c-1}$. 
We assume the overall system output $y=y^C\in{\mathbb R}$ to be real-valued as multiple technical performance indicators (TPIs) could be considered separately or concatenated by weighted mean, etc. 
The simplest example of such a composite system is a \emph{linear chain} $S=S^C\circ\ldots\circ S^2\circ S^1$, where $x\equiv x^1$ is the input into $S^1$ and the output of each component is fed into the next, i.e.\ $x^{c+1}\equiv y^c$. 
Another example is shown in Fig.\ \ref{fig:virtual_system_analysis}, where $x^3$ is concatenated from both outputs $y^1$ and $y^2$. 
We assume the identical compositional structure for the model $M$ with components $M^c$.

\textbf{Validation data.} An essential characteristic of our setup is that neither $S$ nor the subsystem maps $S^c$ are known explicitly, and that ``end-to-end'' measurements $(x,y)$ from the full system $S$ are unavailable (see Sec.\ \ref{sec:introduction}).
Rather, we assume that \emph{validation data} are available only for every subsystem $S^c$, i.e.\ pairs $(x^c_v,y^c_v)$ of inputs $x^c_v$ and corresponding output samples $y^c_v\sim S^c(\cdot|x^c_v)$ ($v=1,\ldots,V^c$).
Such validation data may have been obtained by measuring subsystem $S^c$ in isolation on some inputs $x^c_v$, without needing the full system $S$;
note, the inputs $x^c_v$ do not necessarily follow the distribution from previous components.
In the same spirit, the models $M^c$ may also have been trained from such ``local'' system data; we assume $M^c$, $M$ to be given from the start.

\textbf{Probability distributions.} We aim at probabilistic validation statements, namely that the system fails or violates its requirements only with low probability. 
For this, we assume that $S$ is repeatedly operated in a situation where its inputs come from a distribution $x\sim p_x$, in an i.i.d.\ fashion.
For the example where $S$ is a car, the input $x$ might be a route that typical drivers take in a given city. 
Importantly, we do not assume much knowledge about $p_x$: 
merely a number of samples $x_v\sim p_x$ may be given, or alternatively its distance to the simulation input distribution $q_x=\frac{1}{n_M}\sum_{n=1}^{n_M}\delta_{x^M_n}$; 
here, $\delta_{x^M_n}$ are point measures on the input signals $x^M_n$ on which $M$ is being simulated.
The input distribution $x\sim p_x$ will induce a (joint) distribution $p$ of all intermediate signals $x^c,y^c$ of the composite system $S$ and importantly the TPI output $y=y^C\sim S(x)$. 
Similarly, $M$ generates a joint model distribution $q$ by starting from $q_x$ and sampling through all $M^c$; 
via this simulation, we assume $q$ and all its marginals on intermediate signals $x^c,y^c$ to be available in sample-based form.
The \emph{(true) failure probability} is given by $p_\text{fail}=\int \mathds{1}_{S>\tau}dS(x)dp(x)$, where in this paper we identify a system failure as the TPI exceeding the given threshold $\tau$.
The model failure probability is $q_\text{fail}=\int \mathds{1}_{M>\tau}dM(x)dq(x) \simeq \frac{1}{n_M}\sum_{n}\mathds{1}_{y^M_n>\tau}$, where $y^M_n$ denote sampled model TPI outputs for the inputs $x^M_n$.
It is often useful in our setting to think of a distribution as a set of sample points, and vice versa.

\textbf{Discrepancies.} To track how far the simulation model $M$ diverges from the true system behavior $S$ in our probabilistic setting, we employ discrepancy measures $D$ between probability distributions. 
Such a measure $D$ maps two probability distributions $p,q$ over the same space to a real number, often having some interpretation of distance. 
We consider MMD distances $D=\MMD_k$ \citep{gretton2012kernel}, defined as the RKHS norm $\MMD_k(p,q)=\|p-q\|_k=\big[\int_{x,x'}(p(x)-q(x))k(x,x')(p(x')-q(x'))dxdx'\big]^{1/2}$ w.r.t.\ a kernel $k$ on the underlying space (e.g.\ a squared-exponential or IMQ kernel \citep{gorham2017measuring}). 
Further possibilities include the cosine similarity $\COS_k(p,q)=\langle p,q\rangle_k/\|p\|_k\|q\|_k$ w.r.t.\ a kernel $k$, a Wasserstein distance $D=W_p$ w.r.t.\ a metric on the space, and the total variation norm $D=TV$ \citep{IPM_paper_gretton_2009,sriperumbudur2010non}; however, the latter cannot be estimated reliably from samples.

Specifically, we assume a discrepancy measure $D^{c'\to c}$ to be given\footnote{We will later address how to choose $D$ from a parameterized family $D_\ell$, e.g.\ with different lengthscales $\ell$.} for those pairs $0\leq c'<c\leq C+1$ for which (a sub-tuple of) the output signal $y^{c'}$ is fed into the input $x^c$ (cf.\ the compositional structure above, and where we define $y^{c'=0}\equiv x$ and $x^{c=C+1}\equiv y:=y^C$). 
This $D^{c'\to c}$ acts on probability distributions over the space of such sub-tuples like $y^{c'}|_{c'\to c}$ (or synonymously, $x^c|_{c'\to c}$), which is defined as the signal entries running from $y^{c'}$ to $x^c$. We denote the marginal of $p$ on these signal entries by $p|_{c'\to c}$, and similar $q|_{c'\to c}$ for $q$. 
In the simplest case of a linear chain, $D^{c'\to c}$ with $c'=c-1$ acts on probability distributions such as $p|_{c'\to c}$ over the space of the (full) vectors $y^{c'}=x^c$. 
We omit superscripts $D^{c'\to c}\equiv D$ when clear from the context.

Our method requires (upper bounds on) the discrepancies $D(p|_{0\to c},q|_{0\to c})$ between marginals of the system and model input distributions $p_x,q_x$; 
specifically between the marginal distributions $p|_{0\to c}$ and $q|_{0\to c}$ over those sub-tuples $x|_{0\to c}$ which are input to subsequent components $c$. 
These discrepancies can either be estimated from samples $x_v,x^M_n$ of $p_x,q_x$, see the biased and unbiased estimates for MMD in \citet{gretton2012kernel}[App.\ A.2, A.3], which are accurate up to at most $\sim\sqrt{(1/n_{min})\log(1/\delta)}$ at confidence level $1-\delta$ (where $n_{min}$ denotes the size of the smaller of both sample sets); alternatively, these discrepancies may be directly given or upper bounded. 
These upper bounds are the quantities $B^{0\to c}$ below in Eq.\ (\ref{general-max-discrepancy-objective}). 
No further knowledge of the real-world input distribution $p_x$ is required.

\subsection{Discrepancy Propagation Method}\label{subsec:validation-method}

We now describe the key step in our method to quantify how closely the model's TPI output distribution, which we denote by $q_y\equiv q|_{C\to C+1}$, approximates the actual (but unknown) system output distribution $p_y\equiv p|_{C\to C+1}$. 
We do this by iteratively propagating worst-case discrepancy values through the (directed and acyclic) graph of components $S^c$/$M^c$, using only the available information, in particular the given validation data $(x^c_v,y^c_v)$ on a per-subsystem basis.

\textbf{Discrepancy bound propagation.} The basic idea is to go through the components $c=1,2,\ldots,C$ one-by-one. 
At each step $S^c$, we consider the ``input discrepancies'' $D(p|_{c'\to c},q|_{c'\to c})$ (for $c'<c$), about which we already have information, and propagate this to gain information about the ``output discrepancies'' $D(p_{c\to c''},q|_{c\to c''})$ (for $c''>c$). 
Here, we consider ``information'' in the form of inequalities $D(p|_{c'\to c},q|_{c'\to c})\leq B^{c'\to c}$, i.e.\ the information is the value of the (upper) bound $B^{c'\to c}$. 
Given bounds $B^{c'\to c}$ on the input signal of $S^c$, an upper bound on $D(p|_{c\to c''},q|_{c\to c''})$ for each fixed $c''>c$ can be found by maximizing the latter discrepancy over all (unknown) distributions $p$ that satisfy all the input discrepancy bounds:
\begin{align} 
B^{c\to c''} = \text{maximize}_p~&D(p|_{c\to c''},q|_{c\to c''})\label{general-max-discrepancy-objective}\\
\text{subject to}~~&D(p|_{c'\to c},q|_{c'\to c})\leq B^{c'\to c}~\forall c'<c. \nonumber
\end{align}
Note that the (sample-based) model distribution $q$ and its marginals in (\ref{general-max-discrepancy-objective}) are known and fixed after the simulation $M$ has been run on the input samples $x_n^M$ which constitute $q_x$ (see above).
In contrast, as the actual $p$ is not known, we maximize over all possible system distributions $p$ in (\ref{general-max-discrepancy-objective}) according to the bounds from the previous components $c'$.

It remains to optimize over all possible sets of marginals $p|_{c\to c''},p|_{c'\to c}$ (for all $c'<c$) occurring in (\ref{general-max-discrepancy-objective}). 
Ideally, one would consider all distributions $p(x^c)$ over input signals $x^c$, apply $S^c$ to each $x^c$ to obtain all possible joint distributions $p(x^c,y^c)=p(x^c)S^c(y^c|x^c)$ of in- and outputs, and compute from this all possible sets of marginals $p|_{c\to c''},p|_{c'\to c}$. 
However, this is impossible as we do not know the action of $S^c$ on every possible input $x^c$.
Rather, we merely know about the action of $S^c$ on the validation inputs $x^c_v$, namely that $y^c_v\sim S^c(x^c_v)$ is a corresponding output sample.
We thus consider only the joint distributions $p(x^c,y^c)=p_\alpha$ that can be formed from the given validation data (Fig.\ \ref{input-output-distribution})\footnote{$\delta_{z_0}$ denotes a Dirac point mass at $z=z_0$.}:
\begin{align}\label{eq:p-alpha-joint}
p_\alpha=\sum_{v=1}^{V^c}\alpha_v\delta_{x^c_v}\delta_{y^c_v},
\end{align}
such that the optimization variable becomes now a probability vector $\alpha\in{\mathbb R}^{V^c}$, i.e.\ with nonnegative entries $\alpha_v\geq0$ summing to $\sum_v\alpha_v=1$. 
By restricting to this (potentially skewed) set of joint distributions $p_\alpha$, the exact bound $B^{c\to c''}$ turns into an estimate; for further discussion see Prop.\ \ref{prop:convergence}, which also proposes another possible parametrization $p_\alpha$.

\begin{figure}[t]
	\begin{center}
		\includegraphics[width=0.85\columnwidth]{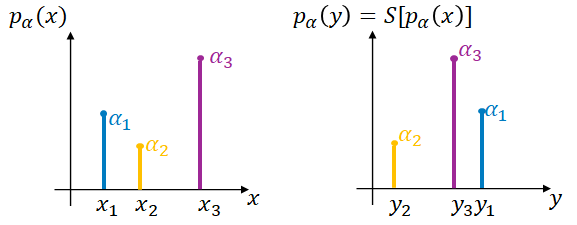}
	\end{center}
	\caption{Illustration of the (marginals of the) joint input-output distribution $p_\alpha$ (\ref{eq:p-alpha-joint}), parameterized by weights $\alpha_v$. Corresponding in-/outputs $x_v,y_v$ have the same weight $\alpha_v$.\label{input-output-distribution}}
\end{figure}

Using ansatz (\ref{eq:p-alpha-joint}), the exact bound propagation (\ref{general-max-discrepancy-objective}) becomes:
\begin{align} 
B^{c\to c''} =\text{max}_\alpha~&D(p_\alpha|_{c\to c''},q|_{c\to c''})\label{max-discrepancy-objective-alpha}\\
\text{s.t.}~~&D(p_\alpha|_{c'\to c},q|_{c'\to c})\leq B^{c'\to c}~~\forall c'<c,\nonumber\\
&\alpha\geq0,~~\sum_v\alpha_v=1. \nonumber
\end{align}
Note that for sample-based distributions like $p_\alpha$ in (\ref{eq:p-alpha-joint}) or $q$, the marginals in this optimization have a similar form, e.g.\ $p_\alpha|_{c\to c''}=\sum_v\alpha_v\delta_{y^c_v|_{c\to c''}}$ or $p_\alpha|_{c'\to c}=\sum_v\alpha_v\delta_{x^c_v|_{c'\to c}}$.

As the discrepancy measures $D$ in (\ref{max-discrepancy-objective-alpha}) are usually convex, this optimization problem is ``almost'' convex: 
All its constraints are convex, however, we aim to \emph{maximize} a convex objective. 
For MMD measures we derive convex (semidefinite) relaxations of (\ref{max-discrepancy-objective-alpha}) by rewriting it with squared MMDs $D(p_\alpha,q)^2$, which are quadratic in $\alpha$ and thus \emph{linear} in a new matrix variable $A=\alpha\alpha^T$; 
this last equality is then relaxed to the semidefinite inequality $A\geq\alpha\alpha^T$ (App.~\ref{app:semidefinite-relaxation}). 
While the relaxation is tight in most instances (App.\ \ref{sec:empirical-tightness}), the number of variables increases from $V^c$ to $\sim\!(V^c)^2/2$, restricting the method to $V^c\lesssim10^3$ validation samples per component. 
In our implementation, we solve these SDPs using the CVXPY package \citep{diamond2016cvxpy}.

\textbf{Bounding the failure probability.} The final step of the preceding \emph{bound propagation} yields an upper bound $B^y:=B^{C\to C+1}$ on the discrepancy $D(p_y,q_y)$ between the (unknown) system TPI output distribution $p_y$ and its model counterpart $q_y$, which is given by samples $y^M_n$. 
We now apply an idea similar to (\ref{eq:p-alpha-joint}) to obtain (a bound on) the system failure probability $p_\text{fail}:=\int_{y>\tau}p_y(y)dy$: 
Rather than maximizing $p_\text{fail}$ over all distributions $p_y$ on ${\mathbb R}\ni y$ subject to the constraint $B^y$, we make the optimization finite-dimensional by selecting grid-points $g_1<g_2<\ldots<g_V\in{\mathbb R}$ and parameterizing $p_y\equiv p_\alpha=\sum_{v=1}^V\alpha_v\delta_{g_v}$, such that $p_\text{fail}=\sum_{v:g_v>\tau}\alpha_v$. 
In practice, we choose an equally-spaced grid in an interval $[g_\text{min},g_\text{max}]\subset{\mathbb R}$ that covers the ``interesting'' or ``plausible'' TPI range, such as the support of $q_y$ as well as sufficient ranges below and above the threshold $\tau$. 
The size of the optimization problem corresponds to the number of grid-points $V$, so $V\!\simeq\!10^3$ is easily possible here.

With this, our final upper bound $p_\text{fail}\leq F_\text{max}$ on the failure probability becomes the following convex program:
\begin{align}\label{max-failure-objective-alpha}
F_\text{max}~ = \text{max}_\alpha~&\sum_{v:g_v>\tau}\alpha_v\\
\text{s.t.}~~&D(p_\alpha,q_y)\leq B^y,~~\alpha\geq0,~~\sum_v\alpha_v=1. \nonumber
\end{align}
One can obtain better (i.e.\ smaller) bounds $F_\text{max}$ by restricting $p_\alpha$ further by plausible assumptions: 
(a) \emph{Monotonicity:} 
When bounding a tail probability, i.e.\ $p_\text{fail}$ is expected to be small, it may be reasonable to assume that $p_y$ is monotonically decreasing beyond some tail threshold $\tau'$. 
For an equally-spaced grid this adds constraints $\alpha_v\leq\alpha_{v-1}$ for all $v$ with $g_v\geq\tau'$ to (\ref{max-failure-objective-alpha}); 
we always assume this with $\tau':=\tau$. 
(b) \emph{Lipschitz condition:} 
To avoid that $p_\alpha$ becomes too ``spiky'', we pose a Lipschitz condition $|\alpha_{v+1}-\alpha_v|\leq\Lambda_\text{max}|g_{v+1}-g_v|$ with a constant $\Lambda_\text{max}$ estimated from the set of outputs $y^M_n$. 
See also App.\ \ref{app:violation-optimization}.

Note that our final bound $F_\text{max}$ is a probability, whose interpretation is independent of the chosen discrepancy measures, kernels, or lengthscales. 
We can thus select these ``parameters'' by minimizing the finally obtained $F_\text{max}$ over them. 
We do this using Bayesian optimization \citep{frohlich2020noisy}.

We summarize our full discrepancy propagation method to obtain a bound $F_\text{max}$ on the system's failure probability in Algorithm \ref{algorithm:DPBound}, which we refer to as \texttt{DPBound}.

\begin{algorithm}
	\caption{\texttt{DPBound}}
	\label{algorithm:DPBound}
	\begin{algorithmic}[1]
		\STATE {\bfseries Input:} compositional structure of $S$;  composite simulation model $M$; discrepancy measures $D$; validation data $(x_v^c,y_v^c)$; simulation input samples $\{x_n^M\}\equiv q_x$; either \textit{(a)} upper bounds $B^{0\to c}$ on input discrepancies or \textit{(b)} samples $x_v\sim p_x$ from the real-world input distribution.
		
		\STATE Run $M$ on all $x_n^M$; collect all intermediate signals $x^c_n,y^c_n$ to build the sample-based marginals $q|_{c'\to c}$ of $q$.
		
		\STATE In case \textit{(b)}, estimate $B^{0\to c}$ from $p_x\simeq\{x_v\}$ and $q_x$.
		
		\FOR{$c=1,\ldots,C$}
		\FOR{every $c''=c+1,\ldots,C+1$ connected to $c$}
		\STATE Compute $B^{c\to c''}$ via Eq.\ (\ref{max-discrepancy-objective-alpha}) (or via App.\ \ref{app:semidefinite-relaxation}).
		\ENDFOR
		\ENDFOR
		
		\STATE Using the thus obtained $B^y:=B^{C\to C+1}$, compute the final bound $F_\text{max}$ via Eq.\ (\ref{max-failure-objective-alpha}) (or via App.\ \ref{app:violation-optimization}).
	\end{algorithmic}
\end{algorithm}

\textbf{Upper bound property.} We replaced the optimization over all possible system distributions $p$ in (\ref{general-max-discrepancy-objective}) by the distributions $p_\alpha$ from (\ref{eq:p-alpha-joint}) due to the limited system validation data and to make the optimization tractable. 
This restricted and possibly skewed $p_\alpha$ can potentially cause $B^{c\to c''}$ and ultimately $F_\text{max}$ from (\ref{max-discrepancy-objective-alpha}),(\ref{max-failure-objective-alpha}) to not be true upper bounds on $D$ or even the system's (unknown) failure probability $p_\text{fail}$, although the worst-case tendency of the maximizations alleviates the issue. 
We investigate this in the experiments (Sec.\ \ref{sec:reli-benchm-eval}), and in the following proposition we state conditions under which (\ref{max-discrepancy-objective-alpha}),(\ref{max-failure-objective-alpha}) are upper bounds:
\begin{proposition}\label{prop:convergence}
	Suppose that for each component $c=1,\ldots,C$: 
	{\it(i)} the validation inputs $x^c_v$ cover the space of occurring inputs into $S^c$;
	{\it(ii)} (necessary only for components $S^c$ having stochastic output) the $\delta_{y^c_v}$ in the defining equation of $p_\alpha$ (Eq.\ (\ref{eq:p-alpha-joint})) is replaced by the system output distribution $S^c(x^c_v)$ (represented e.g.\ by samples or its kernel-mean embedding); 
	{\it(iii)} the grid $\{g_v\}$ covers the occurring TPI values (e.g.\ discrete and bounded). Then $p_\text{fail}\leq F_\text{max}$, where $F_\text{max}$ is defined by the computations in Eqs.\ (\ref{max-discrepancy-objective-alpha}) and (\ref{max-failure-objective-alpha}) (or alternatively, by the convex relaxations and forms in Apps.\ \ref{app:semidefinite-relaxation} and \ref{app:violation-optimization}).
\end{proposition}

App.\ \ref{app:proof-proposition} gives a proof as well as an additional limit statement about the $B^{c\to c''}$ and $F_\text{max}$ in the more realistic setting of increasingly dense inputs $x^c_v$ and approximations of $S^c(x^c_v)$.

\subsection{Failure Bound via Surrogate Model}\label{sec:uncertainty-wrapper-method}
As an alternative, more heuristic but simpler, 
baseline method to estimate an upper bound $F_\text{max}$ on the failure probability $p_\text{fail}$, we introduce a sampling-based method that also operates on the available data only. 
The underlying concept in quantifying model accuracy is similar to the one from \citep{jiang2021assessing} and can also be thought of one particular form of error modeling \citep{riedmaier2021unified}. 

The general idea is to train an additional ``surrogate'' model $M'^c$ for each system component $S^c$, and employ the resulting composite $M'$ to estimate how far $M$ deviates from the real $S$ in terms of TPI outputs. 
This is possible because for training $M'^c$ we use the available validation data $(x^c_v,y^c_v)$ measured from $S^c$. 
We take Gaussian processes (GPs) for $M'^c$, but other probabilistic models like normalizing flows or deterministic ones like neural networks are possible as well.
Choosing $M'^c$ from a different model class than $M^c$ will generally lead to a more conservative estimate $F_\text{max}$. 

After training the $M'^c$, we run the resulting composite surrogate model $M'$ on the simulation input samples $x^M_n$ (which make up $q_x$) to obtain TPI output samples $y'^M_{n,i}$ (with $k$ repetitions $i=1,\ldots,k$), in the same way  the given model $M$ can generate outputs $y^M_{n,i}$ from given $x^M_n$. 
By comparing the $y'^M_{n,i}$ to the $y^M_{n,i}$ we obtain a heuristic estimate of the error of $M$ in simulating the actual TPI of system $S$, see \citep{jiang2021assessing}. 
Concretely in this paper, we use the averaged outputs $y^M_n:=(1/k)\sum_i y^M_{n,i}$ and similarly $y'^M_n$, taking as the simulation error $\Delta$ a high quantile (here, 95\%) of the (signed or absolute) deviations:
\begin{align}\label{eq:delta-surr-model}
\Delta=\text{quantile}_{0.95}[\{y'^M_n-y^M_n\}].
\end{align}
The 100\%-quantile $\max_n(y'^M_n-y^M_n)$ would lead to more conservative bounds $F_\text{max}$. For an illustration see App.\ \ref{sec:signal_propagation_illustration}.

Finally, we estimate an upper bound on $p_\text{fail}$ by including a safety margin $\Delta$ before the threshold $\tau$:
\begin{align}\label{eq:Fmax_estimate_UW}
F_\text{max}:=\int_{\tau-\Delta}^{\infty}q(y)dy=\frac{1}{n_M\cdot k}\sum_{n,i}\mathds{1}_{y^M_{n,i}>\tau-\Delta}.
\end{align}
Note that this method \emph{cannot} account for discrepancies between the (unknown) system input distribution $p_x$ w.r.t.\ which we would like to bound the failure probability $p_\text{fail}$, and the given simulation inputs $x^M_n$ that make up $q_x$, see Sec.\ \ref{sec:experiments}. Rather, the method can be expected to work well only when the sample-based $q_x$ is close to the actual $p_x$. Furthermore, if simulation model and surrogate model share unjustified modeling assumptions, an agreement between the two models might mask differences to the actual system.

\section{EXPERIMENTS}\label{sec:experiments}
We evaluate the proposed method on $8$ benchmark systems in Sec.~\ref{sec:reli-benchm-eval}.\footnote{Our organization is carbon neutral. Therefore, all its activities including research activities (e.g., compute clusters) no longer leave a carbon footprint.} 
For each system, we create $4$ configurations where the simulation models and/or the simulation input distributions differ.
These configurations are illustrated on an artificial example in Sec.~\ref{sec:experiment_linear_use-case}. 

\subsection{Illustration: Gaussian Models}\label{sec:experiment_linear_use-case}
As an exemplary validation problem, consider the following one-component, one-dimensional setup: 
For a linear system  $S:\R\rightarrow\R$ with $S(x)=w_S x + b_S$, we want to assess the failure probability in (\ref{eq:failure_prob}) by means of a linear model $M(x) = w_M x + b_M $. 
Both $S$ and $M$ are stimulated with samples from Gaussian distributions $p_x=\mathcal{N}(\mu_p, \sigma_p^2)$ and $q_x=\mathcal{N}(\mu_q, \sigma_q^2)$, respectively.
We can control the accuracy of $M$ and its input distribution $q_x$ separately, thereby investigating their impacts on the estimated failure probability.

Specifically, we analyze the following two configurations: (a) $M$ and $S$ differ in $b_M\neq b_S$, but they  receive identical inputs $q_x=p_x$ (\emph{Misfit Model \& Perfect Input}); and (b) the model $M$ is identical to $S$, but their respective input distributions $q_x\neq p_x$ differ (\emph{Perfect Model \& Biased Input}).
Under these two configurations, \texttt{DPBound} is illustrated for a single propagation step in Fig.~\ref{fig:linear_use_case_illustration}, where the marginal output distributions of $M$ and $S$ differ (top row marginals in brown resp.\ blue).
However, the output discrepancy bound $B^{1\to2}$ from Eq.\ (\ref{max-discrepancy-objective-alpha}) originates differently in the two configurations: In Fig.\ \ref{fig:b}, the input discrepancy $B^{0\to1}$ 
vanishes due to identical inputs $p_x=q_x$.
The output bound (\ref{max-discrepancy-objective-alpha})  thus directly reflects the difference between the output marginals $S(q_x)$ and $M(q_x)$, without optimization since the constraint forces the weights $\alpha$ to be uniform so that $p_\alpha=q_x$ 
(here, we assumed as validation inputs the $q_x$-samples for simplicity). 

\begin{figure}[t]
	\begin{center}
		\subfigure[]{\label{fig:b}\includegraphics[width=0.235\textwidth]{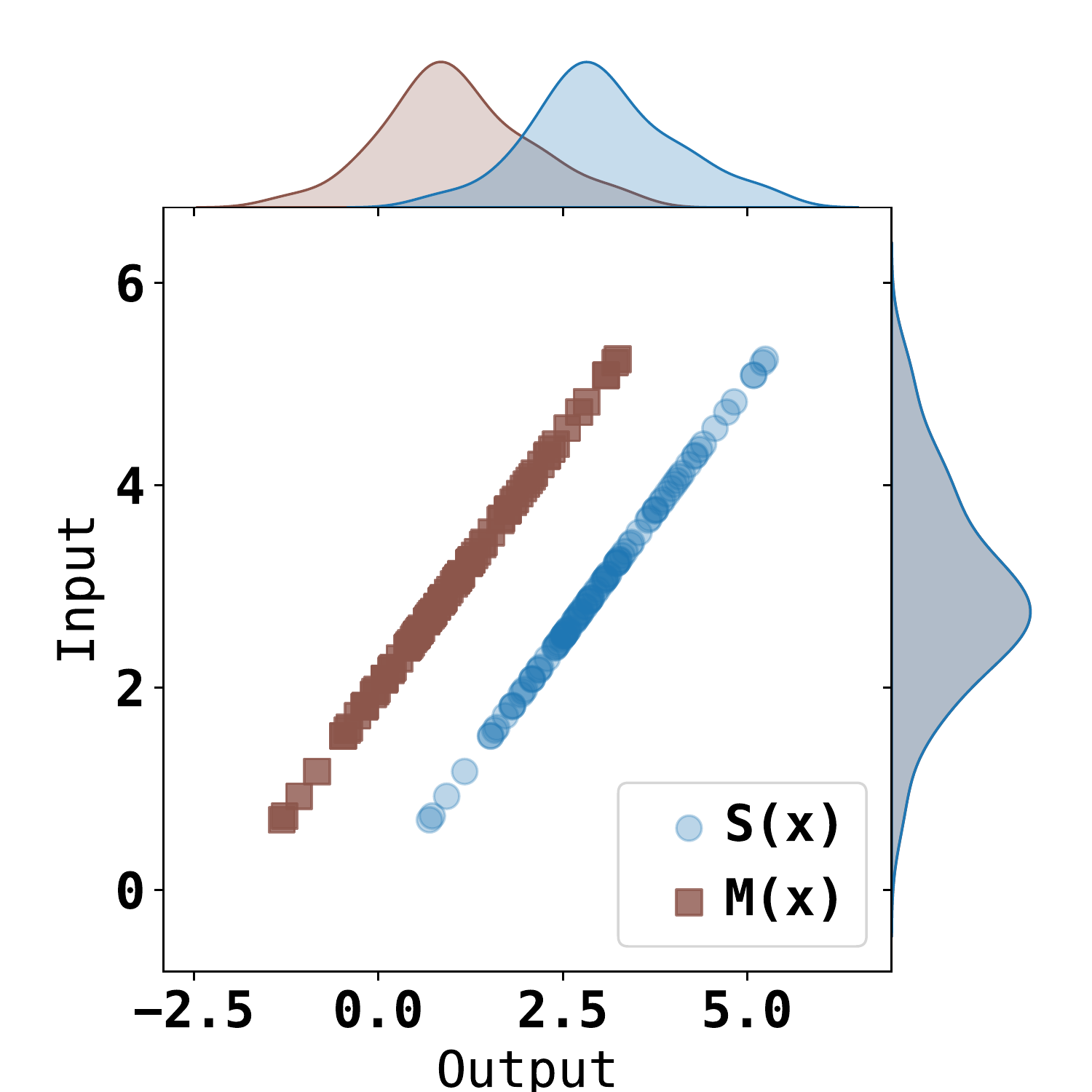}}
		\subfigure[]{\label{fig:a}\includegraphics[width=0.235\textwidth]{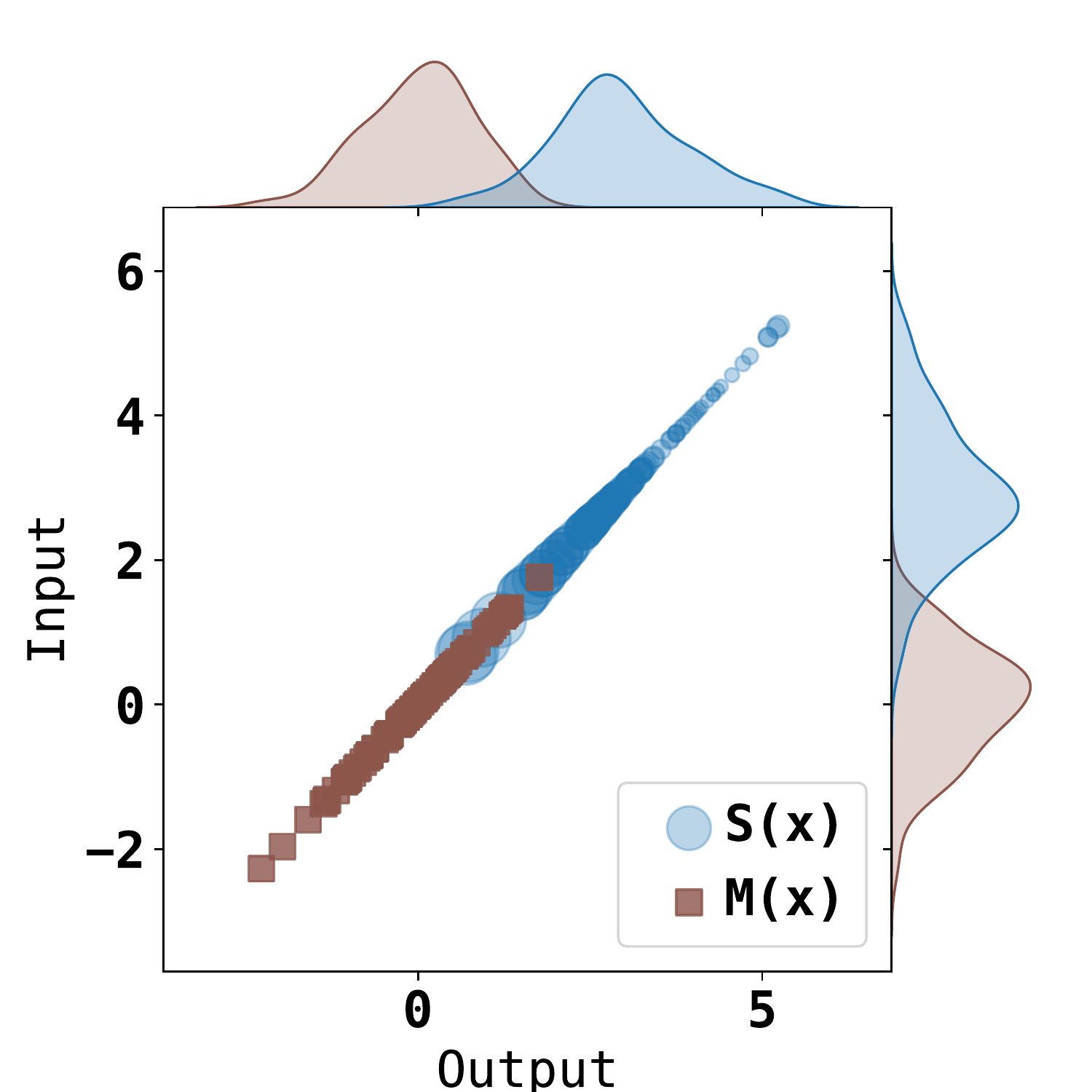}}
	\end{center}
	\caption{Illustration of \texttt{DPBound} for a linear mapping between (samples from) Gaussian signals.
		{\textbf{(a)}} There is model mismatch $M\neq S$, but the input distribution $q_x=p_x$ is perfect.
		{\textbf{(b)}} $M=S$ is a perfect model, but the model input distribution $q_x\neq p_x$ is biased w.r.t.\ the real world.
		The computed weights $\alpha_v$ from Eqs.\ (\ref{eq:p-alpha-joint}),(\ref{max-discrepancy-objective-alpha}) are depicted by the size of the blue $S(x)$-markers ($\alpha$ is uniform in case (a)).\label{fig:linear_use_case_illustration}}
\end{figure}

In Fig.\ \ref{fig:a}, however, the bound on the output discrepancy stems solely 
from the non-zero input discrepancy $B^{0\to1}=MMD(p_x,q_x)>0$, rather than from any difference between $M$ and $S$:
\texttt{DPBound} finds in Eq.\ (\ref{max-discrepancy-objective-alpha}) the worst-case weighted output distribution $p_\alpha$ consistent with the input discrepancy $B^{0\to1}$. 
The output bound is then the nonzero difference between this $p_\alpha$ and $M(q_x)$, even though both $p_\alpha$ and $M(q_x)$ were built with outputs from $S=M$.

In this way, our method \texttt{DPBound} can account for both model misfits and biased inputs. 
The latter is not true of the \texttt{SurrModel} method from Sec.\ \ref{sec:uncertainty-wrapper-method}, as it ignores the real-world distribution $p_x$ and can thus fail for biased inputs $q_x\neq p_x$. For further details and illustrations, see App.\ \ref{appendix:illustrations}.

\subsection{Reliability benchmark evaluation}
\label{sec:reli-benchm-eval}
We now demonstrate the feasibility of our discrepancy propagation method in a reliability benchmark.
\begin{table*}
\caption{Bounds on the failure probability (in \%, with standard deviations from 5 repetitions) delivered by the three compared methods for each benchmark problem under the four simulation configurations \emph{Perfect vs.\ Misfit Model} and \emph{Perfect vs.\ Biased Input}. Each problem has been normalized such that the ground-truth failure probability is 1\%. Also shown (in bold) is the ratio of invalid bounds (i.e.\ bounds below 1\%) delivered by each method among the 40 runs per configuration.\label{tab:results_benchmark}}

\begin{center}
\footnotesize
\resizebox{0.90\textwidth}{!}{ 
\begin{tabular}{l|l|rrr|rrr}
    &            & \multicolumn{3}{c|}{\textbf{Perfect Model}}  & \multicolumn{3}{c}{\textbf{Misfit Model}}\\
           & \cellcolor{verylightgray} Problem &   \cellcolor{verylightgray} DPBound &       \cellcolor{verylightgray} MCCP &  \cellcolor{verylightgray} SurrModel &     \cellcolor{verylightgray} DPBound &       \cellcolor{verylightgray} MCCP &  \cellcolor{verylightgray} SurrModel \\
\midrule
                   & \multicolumn{7}{c}{\cellcolor{veryverylightgray}{Single Component}}  \\
    \multirow{9}{*}{\textbf{\rot{Perfect Input}}} 
          & Borehole &  1.10 $\pm$ 0.2 & 1.87 $\pm$ 0.3 & 3.04 $\pm$ 1.3 &   1.75 $\pm$ 1.7 & 1.06 $\pm$ 0.4 & 3.48 $\pm$ 1.6 \\
           & Branin & 19.98 $\pm$ 3.3 & 2.86 $\pm$ 0.6 & 1.84 $\pm$ 0.5 &  18.87 $\pm$ 2.8 & 2.76 $\pm$ 0.8 & 1.80 $\pm$ 0.7 \\
       & Four Branch &  23.1 $\pm$ 1.0 & 2.08 $\pm$ 0.4 & 1.48 $\pm$ 0.6 &  22.07 $\pm$ 0.7 & 1.07 $\pm$ 0.2 & 0.84 $\pm$ 0.6 \\
                   & \multicolumn{7}{c}{\cellcolor{veryverylightgray}{Multiple Components}}  \\
     & Chained Solvers & 14.92 $\pm$ 2.7 & 2.08 $\pm$ 0.4 & 1.28 $\pm$ 0.6 &  15.21 $\pm$ 2.2 & 2.18 $\pm$ 0.6 & 1.60 $\pm$ 0.4 \\
        & Borehole &  8.94 $\pm$ 6.1 & 2.03 $\pm$ 0.5 & 3.12 $\pm$ 1.9 &   7.51 $\pm$ 2.6 & 2.06 $\pm$ 0.9 & 3.52 $\pm$ 1.6 \\
       & Branin & 17.91 $\pm$ 6.8 & 2.71 $\pm$ 0.5 & 1.56 $\pm$ 0.4 &   16.6 $\pm$ 6.3 & 2.82 $\pm$ 0.2 & 1.88 $\pm$ 0.4 \\
   & Four Branch &  36.4 $\pm$ 3.1 & 2.18 $\pm$ 0.8 & 1.52 $\pm$ 0.7 &  35.82 $\pm$ 3.2 & 0.81 $\pm$ 0.2 & 0.56 $\pm$ 0.3 \\
& Controlled Solvers & 10.39 $\pm$ 4.6 & 1.97 $\pm$ 0.7 & 0.92 $\pm$ 0.5 &  11.01 $\pm$ 4.3 & 1.91 $\pm$ 0.7 & 0.88 $\pm$ 0.5 \\

\hline
    &    \textbf{\% Invalid Bounds} & \textbf{5.0} & \textbf{0.0} & \textbf{17.5} & \textbf{0.0} & \textbf{22.5} & \textbf{27.5}\\
\midrule
    \multirow{9}{*}{\textbf{\rot{Biased Input}}} 
    & \multicolumn{7}{c}{\cellcolor{veryverylightgray}{Single Component}}  \\
   &       Borehole & 16.57 $\pm$ 8.3 & 0.80 $\pm$ 0.3 & 0.88 $\pm$ 0.8 & 15.19 $\pm$ 10.5 & 0.60 $\pm$ 0.0 & 0.76 $\pm$ 0.8 \\
   &         Branin &  92.7 $\pm$ 3.3 & 0.73 $\pm$ 0.3 & 0.08 $\pm$ 0.2 &  93.35 $\pm$ 2.9 & 2.01 $\pm$ 0.8 & 1.00 $\pm$ 0.6 \\
   &    Four Branch & 22.69 $\pm$ 0.9 & 1.98 $\pm$ 0.5 & 1.28 $\pm$ 0.5 &  21.96 $\pm$ 1.0 & 0.88 $\pm$ 0.2 & 0.60 $\pm$ 0.3 \\
                   & \multicolumn{7}{c}{\cellcolor{veryverylightgray}{Multiple Components}}  \\
     & Chained Solvers & 26.39 $\pm$ 1.7 & 0.74 $\pm$ 0.2 & 0.08 $\pm$ 0.1 &   26.7 $\pm$ 1.7 & 1.39 $\pm$ 0.7 & 0.56 $\pm$ 0.5 \\
    & Borehole & 20.96 $\pm$ 9.0 & 0.93 $\pm$ 0.4 & 0.76 $\pm$ 0.8 &  15.89 $\pm$ 7.5 & 0.60 $\pm$ 0.0 & 0.44 $\pm$ 0.3 \\
     &   Branin & 89.52 $\pm$ 6.1 & 0.81 $\pm$ 0.2 & 0.12 $\pm$ 0.1 &  89.52 $\pm$ 6.1 & 0.91 $\pm$ 0.5 & 0.24 $\pm$ 0.3 \\
   & Four Branch & 35.83 $\pm$ 3.2 & 1.96 $\pm$ 0.9 & 1.28 $\pm$ 0.7 &  35.42 $\pm$ 2.9 & 0.74 $\pm$ 0.2 & 0.44 $\pm$ 0.2 \\
& Controlled Solvers & 13.36 $\pm$ 4.1 & 0.60 $\pm$ 0.0 & 0.00 $\pm$ 0.0 &  13.51 $\pm$ 5.9 & 0.67 $\pm$ 0.2 & 0.04 $\pm$ 0.1 \\
\hline
    &    \textbf{\% Invalid Bounds} & \textbf{0.0} & \textbf{67.5} & \textbf{70.0} & \textbf{0.0} & \textbf{67.5} & \textbf{75.0} \\
    \bottomrule
\end{tabular}
}
\end{center}
\end{table*}

\paragraph{Compared benchmark systems.}
The performance of \texttt{DPBound} is evaluated on 3 single-component and 5 multi-component problems from the reliability and uncertainty propagation literature.
These problems are briefly summarized in the following table, see App.\ \ref{sec:deta-exper-sect} for more details:
\begin{center}
\scriptsize
\begin{tabular}{ |l|c|c| } 
\hline
Benchmark System & Input Dim. & Components \\
\hline
    {\bf{Controlled Solvers}}~\citep{sanson2019systems} & 16 & 4 \\
    {\bf{Chained Solvers}}~\citep{sanson2019systems} & 1 & 2 \\
    {\bf{Borehole}}~\citep{sim_bench_website} & 8 & 1 / 5 \\
    {\bf{Branin}}~\citep{sim_bench_website} & 2 & 1 / 3 \\
    {\bf{Four Branch}}~\citep{UQworld} & 2 & 1 / 4 \\
\hline
\end{tabular}
\end{center}

For the evaluation, we set the threshold $\tau$ on the scalar output for each of those problems such that the ground-truth failure probability $\textrm{Pr}_{x\sim p_x}[S(x)>\tau]=1\%$ (see Eq.\ (\ref{eq:failure_prob})).

We evaluate each system under four different simulation configurations (cf.\ Sec.\ \ref{sec:experiment_linear_use-case}): As simulation models, we take either \emph{Perfect Models} $M^c=S^c$ or GP-based \emph{Misfit Models} $M^c\neq S^c$; as simulation input distribution, we take either \emph{Perfect Input} $q_x=p_x$ or \emph{Biased Input} $q_x\neq p_x$ (App.\ \ref{sec:deta-exper-sect}).

\textbf{Compared virtual validation methods.} We compare our failure probability bound with two alternative methods:\\
{\bf{\texttt{DPBound} \textbf{(ours):}}}
Failure bound $F_\text{max}$ calculated by propagating MMD-based bounds according to Algorithm \ref{algorithm:DPBound}.\\ 
{\bf{\texttt{MCCP}}\textbf{:}}
95\%-confidence Clopper-Pearson bound on the failure probability, calculated on binary Monte-Carlo samples obtained by thresholding the output of the simulation model. \\ 
{\bf{\texttt{SurrModel}}\textbf{:}}
Bound from Eq.\ (\ref{eq:Fmax_estimate_UW}), obtained by accounting for the difference between the simulation and a GP-based surrogate model learned on the validation data (Sec.\ \ref{sec:uncertainty-wrapper-method}).

\paragraph{Experimental results.} The obtained results are summarized in Tab.\ \ref{tab:results_benchmark}. 
We first focus on the validity of the methods (bold numbers in Tab.\ \ref{tab:results_benchmark}): 
In this regard one can see that, in the ``Perfect Input--Perfect Model'' setting the methods produce generally valid bounds, with at most 17.5\% invalidness for \texttt{SurrModel}. 
In the more challenging and realistic settings with misfit and/or input bias, however, the invalidness ratios for \texttt{MCCP} and \texttt{SurrModel} increase beyond acceptable levels, especially for \emph{Biased Input} with invalidness reaching up to 75\%. 
\texttt{DPBound} on the other hand remains perfectly valid under both \emph{Misfit Model} and \emph{Biased Input}, in line with the illustration in Sec.\ \ref{sec:experiment_linear_use-case}.

To understand why \texttt{MCCP} and \texttt{SurrModel} have challenges with the \emph{Biased Input} setting, notice that these methods disregard the actual system inputs $p_x$, instead relying solely on the simulation inputs $q_x$ 
without any means of dealing with a potential discrepancy between both distributions.
When the simulation input distribution is biased towards significantly lower simulated TPI values, the delivered bounds can then be invalid (for an illustration see App.\ \ref{sec:signal_propagation_illustration}). 
\texttt{DPBound} on the other hand natively accounts for this input discrepancy through the initial bounds $B^{0\to c}$, which in our experiments are estimated via samples from $p_x$, $q_x$ (see end of Sec.\ \ref{sec:setup}). 
In addition to this shortcoming of ignoring input discrepancies, \texttt{MCCP} remains unaware of any potential \emph{Misfit Model}, as it completely ignores the system $S$ and its validation data. 
        This explains the jump in invalidness from 0.0\% to 22.5\% when isolating this effect in the \emph{Perfect Input} setting. 
Note, while \texttt{MCCP} is not expected to produce upper bounds in 100\% of cases due to its 95\%-confidence specification, it stays significantly below the 95\% promise (App.\ \ref{app:MCCP99}). 

Although \texttt{DPBound} provides valid upper bounds $F_\text{max}$ in almost all cases, it can sometimes still underestimate, as happened in two validation runs for Borehole(Single) under \emph{Perfect Input--Perfect Model} with bounds close to the 1\% ground-truth.  
To explain how this can happen, note that only under dense sampling conditions is \texttt{DPBound} expected to be perfectly valid (Prop. \ref{prop:convergence} and App.\ \ref{app:proof-proposition}). 
In any case, \texttt{DPBound} shows high validity overall, while the two competing methods' likelihood for falsely positive validation is certainly too high for trustable statements.

Summarizing these validity results, of the three compared methods, only \texttt{DPBound} should be further considered to be viable at all as a reliable validation method.

Despite its high validity, \texttt{DPBound} delivers bounds $F_\text{max}$ that are mostly far from trivial (i.e.\ much below 100\%), also in the challenging misfit and/or biased settings. 
These bounds in conjunction with their high validity thus yield useful information, which can serve as a basis for
extended validation approaches (see Conclusion). 
That \texttt{MCCP} and \texttt{SurrModel} often produce much smaller or ``tighter''\footnote{Note, the ``tightness'' of the bounds can be read off from Tab.~\ref{tab:results_benchmark} by subtracting from each bound the ground-truth value of 1\%.} bounds is clearly no advantage per se, as these are often invalid and thus misleading in validation (see above). 
We did not focus on the thightness evaluation because our main evaluation criterion was the rate of falsely positive validations, which already excluded both competing methods in safety-relevant situations.

While it remains for future work to combine \texttt{DPBound}'s non-trivial and reliable bounds with other validation approaches, our investigation here constitutes the first one into the validity of validation methods for the component-wise setting, establishing \texttt{DPBound} as a viable candidate.

\section{CONCLUSION}\label{sec:conclusion}
Validating complex composite systems is a notoriously difficult 
task, e.g.\ validating the performance of autonomously driving vehicles.
Instead of expensively testing the system in the real world, simulations can reduce the validation effort \citep{wong2020testing}.
However, it is hard to assess the effect of simulation model inaccuracies or 
input data shifts on the validation target, especially for composite systems.
We have developed a method to estimate an upper bound on the system failure probability, as underestimation of failure rates is typically much more costly than overestimation. 
Our method assumes that a simulation model as well as measurement data for each subsystem are available. 
Our evaluations show that 
the obtained bounds are useful and valid in general, with theoretical guarantees in the large-data limit (Prop.\ \ref{prop:convergence}). 

Due to its individual-component nature, our method is especially fit to use when only one component in an already deployed system changes, e.g.\ a sensor or the software controller in an autonomous driving system.
Although the values computed for the bounds by our propagation method are larger than what would typically be required for safety-relevant applications, they still yield useful information, for example by acting as a safe-guard before entering an expensive real-world testing phase. Continuing the proposed avenue of research may ultimately spawn validation with immensely reduced number of real-world test runs.
For the future, it remains to explore the method in higher-dimensional situations, possibly by extending our parameterization of the joint distribution $p_\alpha$.
While the presented method was derived for static signals, it can be extended to dynamic systems and models by replacing the static signals with embeddings of time-series signals \citep{morrill2020generalised}.
Exploring the proposed method in these regimes, which often include feedback loops, is subject to future research.

\section*{ACKNOWLEDGEMENTS}
This research was supported by the German Federal Ministry for Economic Affairs and Climate Action under the joint project ``KI-Embedded'' (grant no.\ 19I21043A).

\bibliography{references}

\begin{thebibliography}{}

\bibitem[Agrawal et~al., 2019]{agrawal2019differentiable}
Agrawal, A., Amos, B., Barratt, S., Boyd, S., Diamond, S., and Kolter, J.~Z.
  (2019).
\newblock Differentiable convex optimization layers.
\newblock {\em Advances in neural information processing systems}, 32.

\bibitem[Astudillo and Frazier, 2019]{astudillo2019bayesian}
Astudillo, R. and Frazier, P. (2019).
\newblock Bayesian optimization of composite functions.
\newblock In {\em International Conference on Machine Learning}, pages
  354--363. PMLR.

\bibitem[Astudillo and Frazier, 2021]{astudillo2021bayesian}
Astudillo, R. and Frazier, P. (2021).
\newblock Bayesian optimization of function networks.
\newblock {\em Advances in Neural Information Processing Systems},
  34:14463--14475.

\bibitem[Au and Beck, 2001]{au2001estimation}
Au, S.-K. and Beck, J.~L. (2001).
\newblock Estimation of small failure probabilities in high dimensions by
  subset simulation.
\newblock {\em Probabilistic engineering mechanics}, 16(4):263--277.

\bibitem[Bect et~al., 2012]{bect2012sequential}
Bect, J., Ginsbourger, D., Li, L., Picheny, V., and Vazquez, E. (2012).
\newblock Sequential design of computer experiments for the estimation of a
  probability of failure.
\newblock {\em Statistics and Computing}, 22(3):773--793.

\bibitem[Belcastro and Belcastro, 2003]{belcastro2003validation}
Belcastro, C. and Belcastro, C. (2003).
\newblock On the validation of safety critical aircraft systems, part {I}: An
  overview of analytical \& simulation methods.
\newblock In {\em AIAA Guidance, Navigation, and Control Conference and
  Exhibit}, page 5559.

\bibitem[Chau et~al., 2021]{chau2021bayesimp}
Chau, S.~L., Ton, J.-F., Gonz{\'a}lez, J., Teh, Y., and Sejdinovic, D. (2021).
\newblock {BayesIMP}: Uncertainty quantification for causal data fusion.
\newblock {\em Advances in Neural Information Processing Systems},
  34:3466--3477.

\bibitem[Damianou and Lawrence, 2013]{damianou2013deep}
Damianou, A. and Lawrence, N.~D. (2013).
\newblock Deep {G}aussian processes.
\newblock In {\em Artificial intelligence and statistics}, pages 207--215.
  PMLR.

\bibitem[Diamond and Boyd, 2016]{diamond2016cvxpy}
Diamond, S. and Boyd, S. (2016).
\newblock {CVXPY}: {A} {P}ython-embedded modeling language for convex
  optimization.
\newblock {\em Journal of Machine Learning Research}, 17(83):1--5.

\bibitem[Du and Hu, 2012]{du2012first}
Du, X. and Hu, Z. (2012).
\newblock First order reliability method with truncated random variables.
\newblock {\em Journal of Mechanical Design}, 134(9).

\bibitem[Dubourg et~al., 2013]{dubourg2013metamodel}
Dubourg, V., Sudret, B., and Deheeger, F. (2013).
\newblock Metamodel-based importance sampling for structural reliability
  analysis.
\newblock {\em Probabilistic Engineering Mechanics}, 33:47--57.

\bibitem[Friedman et~al., 2021]{friedman2021adaptive}
Friedman, S., Jakeman, J.~D., Eldred, M.~S., Tamellini, L., Gorodestky, A., and
  Allaire, D. (2021).
\newblock Adaptive resource allocation for surrogate modeling of systems
  comprised of multiple disciplines with varying fidelity.
\newblock Technical report, Sandia National Lab (SNL-NM), Albuquerque, NM
  (United States).

\bibitem[Fr{\"o}hlich et~al., 2020]{frohlich2020noisy}
Fr{\"o}hlich, L., Klenske, E., Vinogradska, J., Daniel, C., and Zeilinger, M.
  (2020).
\newblock Noisy-input entropy search for efficient robust {B}ayesian
  optimization.
\newblock In {\em International Conference on Artificial Intelligence and
  Statistics}, pages 2262--2272. PMLR.

\bibitem[Gardner et~al., 2018]{gardner2018gpytorch}
Gardner, J.~R., Pleiss, G., Bindel, D., Weinberger, K.~Q., and Wilson, A.~G.
  (2018).
\newblock {GPyTorch}: {B}lackbox matrix-matrix {G}aussian process inference
  with {GPU} acceleration.
\newblock In {\em Advances in Neural Information Processing Systems},
  volume~31.

\bibitem[Gorham and Mackey, 2017]{gorham2017measuring}
Gorham, J. and Mackey, L. (2017).
\newblock Measuring sample quality with kernels.
\newblock In {\em International Conference on Machine Learning}, pages
  1292--1301. PMLR.

\bibitem[Gretton et~al., 2012]{gretton2012kernel}
Gretton, A., Borgwardt, K.~M., Rasch, M.~J., Sch{\"o}lkopf, B., and Smola, A.
  (2012).
\newblock A kernel two-sample test.
\newblock {\em The Journal of Machine Learning Research}, 13(1):723--773.

\bibitem[Hohenbichler et~al., 1987]{hohenbichler1987new}
Hohenbichler, M., Gollwitzer, S., Kruse, W., and Rackwitz, R. (1987).
\newblock New light on first- and second-order reliability methods.
\newblock {\em Structural safety}, 4(4):267--284.

\bibitem[Jiang et~al., 2022]{jiang2021assessing}
Jiang, Y., Nagarajan, V., Baek, C., and Kolter, J.~Z. (2022).
\newblock Assessing generalization of {SGD} via disagreement.
\newblock In {\em 10th International Conference on Learning Representations,
  {ICLR} 2022}.

\bibitem[Kalra and Paddock, 2016]{kalra2016driving}
Kalra, N. and Paddock, S.~M. (2016).
\newblock Driving to safety: How many miles of driving would it take to
  demonstrate autonomous vehicle reliability?
\newblock {\em Transportation Research Part A: Policy and Practice},
  94:182--193.

\bibitem[Kiureghian and Stefano, 1991]{kiureghian1991efficient}
Kiureghian, A.~D. and Stefano, M.~D. (1991).
\newblock Efficient algorithm for second-order reliability analysis.
\newblock {\em Journal of engineering mechanics}, 117(12):2904--2923.

\bibitem[Koopman and Wagner, 2016]{koopman2016challenges}
Koopman, P. and Wagner, M. (2016).
\newblock Challenges in autonomous vehicle testing and validation.
\newblock {\em SAE International Journal of Transportation Safety},
  4(1):15--24.

\bibitem[Kusakawa et~al., 2022]{kusakawa2021bayesian}
Kusakawa, S., Takeno, S., Inatsu, Y., Kutsukake, K., Iwazaki, S., Nakano, T.,
  Ujihara, T., Karasuyama, M., and Takeuchi, I. (2022).
\newblock Bayesian optimization for cascade-type multistage processes.
\newblock {\em Neural Computation}, 34(12):2408--2431.

\bibitem[Lee et~al., 2012]{lee2012novel}
Lee, I., Noh, Y., and Yoo, D. (2012).
\newblock A novel second-order reliability method {(SORM)} using noncentral or
  generalized chi-squared distributions.
\newblock {\em Journal of Mechanical Design}, 134(10).

\bibitem[Marque-Pucheu et~al., 2019]{marque2019efficient}
Marque-Pucheu, S., Perrin, G., and Garnier, J. (2019).
\newblock Efficient sequential experimental design for surrogate modeling of
  nested codes.
\newblock {\em ESAIM: Probability and Statistics}, 23:245--270.

\bibitem[Morrill et~al., 2020]{morrill2020generalised}
Morrill, J., Fermanian, A., Kidger, P., and Lyons, T. (2020).
\newblock A generalised signature method for multivariate time series feature
  extraction.
\newblock {\em arXiv preprint arXiv:2006.00873}.

\bibitem[Muandet et~al., 2017]{muandet2017KME}
Muandet, K., Fukumizu, K., Sriperumbudur, B., and Sch\"olkopf, B. (2017).
\newblock Kernel mean embedding of distributions: A review and beyond.
\newblock {\em Foundations and Trends in Machine Learning}, 10(1):1--141.

\bibitem[Park and Boyd, 2017]{park2017general}
Park, J. and Boyd, S. (2017).
\newblock General heuristics for nonconvex quadratically constrained quadratic
  programming.
\newblock {\em arXiv preprint arXiv:1703.07870}.

\bibitem[Pradlwarter et~al., 2007]{pradlwarter2007application}
Pradlwarter, H., Schueller, G., Koutsourelakis, P.-S., and Charmpis, D.~C.
  (2007).
\newblock Application of line sampling simulation method to reliability
  benchmark problems.
\newblock {\em Structural Safety}, 29(3):208--221.

\bibitem[Riedmaier et~al., 2021]{riedmaier2021unified}
Riedmaier, S., Danquah, B., Schick, B., and Diermeyer, F. (2021).
\newblock Unified framework and survey for model verification, validation and
  uncertainty quantification.
\newblock {\em Archives of Computational Methods in Engineering},
  28(4):2655--2688.

\bibitem[Rubinstein and Kroese, 2004]{rubinstein2004cross}
Rubinstein, R.~Y. and Kroese, D.~P. (2004).
\newblock {\em The Cross-Entropy Method: A Unified Approach to Combinatorial
  Optimization, Monte-Carlo Simulation, and Machine Learning}, volume 133.
\newblock Springer.

\bibitem[Sanson et~al., 2019]{sanson2019systems}
Sanson, F., Le~Maitre, O., and Congedo, P.~M. (2019).
\newblock Systems of {G}aussian process models for directed chains of solvers.
\newblock {\em Computer Methods in Applied Mechanics and Engineering},
  352:32--55.

\bibitem[Sriperumbudur et~al., 2009]{IPM_paper_gretton_2009}
Sriperumbudur, B.~K., Fukumizu, K., Gretton, A., Schölkopf, B., and Lanckriet,
  G.~R. (2009).
\newblock On integral probability metrics, $\phi$-divergences and binary
  classification.
\newblock {\em arXiv preprint arXiv:0901.2698}.

\bibitem[Sriperumbudur et~al., 2010]{sriperumbudur2010non}
Sriperumbudur, B.~K., Fukumizu, K., Gretton, A., Sch{\"o}lkopf, B., and
  Lanckriet, G.~R. (2010).
\newblock Non-parametric estimation of integral probability metrics.
\newblock In {\em 2010 IEEE International Symposium on Information Theory},
  pages 1428--1432.

\bibitem[Surjanovic and Bingham, 2023]{sim_bench_website}
Surjanovic, S. and Bingham, D. (2023).
\newblock Virtual library of simulation experiments: Test functions and
  datasets.
\newblock Retrieved February 13, 2023, from \url{http://www.sfu.ca/~ssurjano}.

\bibitem[UQWorld, 2023]{UQworld}
UQWorld (2023).
\newblock {Four-Branch Function}.
\newblock Retrieved February 13, 2023, from
  \url{https://uqworld.org/t/four-branch-function/59}.

\bibitem[Wang et~al., 2016]{wang2016gaussian}
Wang, H., Lin, G., and Li, J. (2016).
\newblock Gaussian process surrogates for failure detection: A {B}ayesian
  experimental design approach.
\newblock {\em Journal of Computational Physics}, 313:247--259.

\bibitem[Wong et~al., 2020]{wong2020testing}
Wong, K., Zhang, Q., Liang, M., Yang, B., Liao, R., Sadat, A., and Urtasun, R.
  (2020).
\newblock Testing the safety of self-driving vehicles by simulating perception
  and prediction.
\newblock In {\em European Conference on Computer Vision}, pages 312--329.
  Springer.

\bibitem[Xiao et~al., 2022]{xiao2022projection}
Xiao, T., Balasubramanian, K., and Ghadimi, S. (2022).
\newblock A projection-free algorithm for constrained stochastic multi-level
  composition optimization.
\newblock In {\em Advances in Neural Information Processing Systems},
  volume~35.

\bibitem[Xu and Saleh, 2021]{xu2021machine}
Xu, Z. and Saleh, J.~H. (2021).
\newblock Machine learning for reliability engineering and safety applications:
  Review of current status and future opportunities.
\newblock {\em Reliability Engineering \& System Safety}, 211:107530.

\bibitem[Zhang et~al., 2015]{zhang2015first}
Zhang, Z., Jiang, C., Wang, G., and Han, X. (2015).
\newblock First and second order approximate reliability analysis methods using
  evidence theory.
\newblock {\em Reliability Engineering \& System Safety}, 137:40--49.

\end{thebibliography}

\newpage
\onecolumn 
\appendix

\begin{center}
	{\large Supplementary Material for the paper}
	
	{\LARGE Validation of Composite Systems}
	
	{\LARGE by Discrepancy Propagation}
\end{center}

\bigskip

\section{SEMIDEFINITE RELAXATION OF THE BOUND OPTIMIZATION IN EQ.\ (\ref{max-discrepancy-objective-alpha})}\label{app:semidefinite-relaxation}

Let the discrepancy measures in Eq.\ (\ref{max-discrepancy-objective-alpha}) in Sec.\ \ref{subsec:validation-method} be given by MMD (maximum mean discrepancy) with kernels $k^{c\rightarrow c''}$ and $k^{c'\rightarrow c}$, respectively. 
When representing the distributions via samples, we have (see \citep{gretton2012kernel}):
\newcommand{\kout}[2]{\ensuremath{k^{c\rightarrow c''}\left(#1,#2\right) }}
\newcommand{\kin}[2]{\ensuremath{k^{c'\rightarrow c}\left(#1,#2\right) }}
\newcommand{\Kout}{\ensuremath{{\bf{K}}_{c\rightarrow c''}}}
\newcommand{\Kin}{\ensuremath{{\bf{K}}_{c'\rightarrow c}}}
\newcommand{\en}[1]{\ensuremath{{\bf{e}}_{#1}}}
\begin{align}
D(p_\alpha|_{c\rightarrow c''}, q|_{c\rightarrow c''})^2 & = \alpha ^\top \Kout^{VV}\alpha  - 2\alpha^\top \Kout^{VM}\frac{\en{n_M}}{n_M} + \frac{\en{n_M}^\top}{n_M} \Kout^{MM}\frac{\en{n_M}}{n_M}
\end{align}
with kernel matrices (with indices $v,v'=1,\ldots,V^c$, $n,n'=1,\ldots,n_M$)
\begin{align}
\left(\Kout^{VV}\right)_{v,v'} & := \kout{y_v^c}{ y_{v'}^c},\nonumber\\
\left(\Kout^{VM}\right)_{v,n'} & := \kout{y_v^c}{ y_{n'}^{M^c}},\nonumber\\
\left(\Kout^{MM}\right)_{n,n'} & := \kout{y_{n}^{M^c}}{ y_{n'}^{M^c}},\nonumber
\end{align}
and where $\en{d}:=(1,\dots,1)^\top\in{\mathbb R}^d$ denotes the $d$-dimensional all-1's vector, so that $\frac{\en{d}}{d}$ is the uniform probability vector on $d$ elements. Similarly, at the input of component $c$ we have
\begin{align}
D(p_\alpha|_{c'\rightarrow c}, q|_{c'\rightarrow c})^2 & = \alpha ^\top \Kin^{VV}\alpha  - 2\alpha^\top \Kin^{VM}\frac{\en{n_M}}{n_M} + \frac{\en{n_M}^\top}{n_M} \Kin^{MM}\frac{\en{n_M}}{n_M},
\end{align}
where by a slight abuse of notation we define the input kernel matrices as
\begin{align}
\left(\Kin^{VV}\right)_{v,v'} & := \kin{x_v^c}{ x_{v'}^c},\nonumber\\
\left(\Kin^{VM}\right)_{v,n'} & := \kin{x_v^c}{ x_{n'}^{M^c}},\nonumber\\
\left(\Kin^{MM}\right)_{n,n'} & := \kin{x_{n}^{M^c}}{ x_{n'}^{M^c}}.\nonumber
\end{align}

Taken together, we can write (\ref{max-discrepancy-objective-alpha}) in Sec.\ \ref{subsec:validation-method} -- or rather its square -- as the following quadratic optimization problem:
\begin{align}
&(B^{c\rightarrow c''})^2\\
&=\sup_{\{p:D(p_\alpha|_{c'\rightarrow c},q|_{c'\rightarrow c})^2\leq(B^{c'\rightarrow c})^2~~ \forall c'< c\}}D(p_\alpha|_{c\rightarrow c''}, q|_{c\rightarrow c''})^2\\
&=\text{maximize}_\alpha ~\alpha^\top \Kout^{VV}\alpha -2\alpha^\top \Kout^{VM}\frac{\en{n_M}}{n_M}+\frac{\en{n_M}^\top}{n_M} \Kout^{MM}\frac{\en{n_M}}{n_M} \label{eq-no-triang:app-first-term-objective}\\
& \qquad\text{subject to }~\alpha^\top \Kin^{VV}\alpha-2\alpha^\top \Kin^{VM}\frac{\en{n_M}}{n_M}+\frac{\en{n_M}^\top}{n_M} \Kin^{MM}\frac{\en{n_M}}{n_M}\leq(B^{c'\rightarrow c})^2 \quad\forall c'<c\label{eq-no-traing:app-first-term-first-constraint},\\
&\qquad\qquad\qquad~~\en{V^c}^\top\alpha=1,\label{eq-no-triang:equality-constraint}\\
&\qquad\qquad\qquad~~\alpha\geq0,\label{eq-no-traing:app-first-term-last-constraint}
\end{align}
where the vector constraint $\alpha\geq0$ is understood entry-wise.

Unfortunately, the optimization problem (\ref{eq-no-triang:app-first-term-objective}--\ref{eq-no-traing:app-first-term-last-constraint}) is \emph{not} a convex optimization problem because the objective is to maximize a convex function. While heuristic solvers are available, such as the package \texttt{qcqp} \citep{park2017general}, we avoid those heuristic methods, as they do not guarantee a valid upper bound and can be inefficient in computation.

Instead we follow approaches more tailored to the problem, and relax the original problem (\ref{eq-no-triang:app-first-term-objective}--\ref{eq-no-traing:app-first-term-last-constraint}) to obtain an efficiently solvable semidefinite program (SDP), a type of convex optimization problem. 
We follow the ``tightened semidefinite relaxations" from \citet{park2017general}[Secs.\ 3.3, 3.4].

For this, we introduce the (symmetric) matrix variable $A:=\alpha\alpha^\top$ and rewrite the quadratic terms in (\ref{eq-no-triang:app-first-term-objective}) and (\ref{eq-no-traing:app-first-term-first-constraint}) via the matrix traces $\alpha^\top \Kout^{VV}\alpha=\mathrm{Tr}[\Kout^{VV}A]$ and $\alpha^\top \Kin^{VV}\alpha=\mathrm{Tr}[\Kin^{VV}A]$, so that they are now \emph{linear} in the variable $A$. 
Due to (\ref{eq-no-traing:app-first-term-last-constraint}), $A=\alpha\alpha^\top$ is entry-wise nonnegative, which we write as $A\succcurlyeq0$; furthermore, it holds that $\en{V_c}^\top A\en{V_c}=1$ due to (\ref{eq-no-triang:equality-constraint}). 
Also due to (\ref{eq-no-triang:equality-constraint}), we can even recover $\alpha$ from $A=\alpha\alpha^\top$ via $\alpha=A\en{V_c}$. 
Therefore, we will omit our original variable $\alpha$ in favor of the matrix variable $A$ and simply \emph{define} the expression $\alpha:=A\en{V_c}$. 
Rewriting (\ref{eq-no-triang:app-first-term-objective}--\ref{eq-no-traing:app-first-term-last-constraint}) in terms of $A$ and additionally adding the constraints $A=\alpha\alpha^\top$, $\en{V_c}^\top A\en{V_c}=1$ and $A\succcurlyeq0$ gives the same optimum as (\ref{eq-no-triang:app-first-term-objective}--\ref{eq-no-traing:app-first-term-last-constraint}). 
The resulting (rewritten) problem is convex except for the equality constraint $A=\alpha\alpha^\top$. 
We finally relax this constraint to the matrix inequality $A\geq\alpha\alpha^\top$ (where the inequality is understood w.r.t.\ the positive semidefinite order), which \emph{is} convex. 
Writing this matrix inequality in the manifestly convex way (\ref{eq:tightened-SDR-relaxation-A-alpha}) as a semidefinite constraint \citep{park2017general}, we therefore obtain:
\begin{lemma}[Tightened SDP relaxation of (\ref{eq-no-triang:app-first-term-objective}--\ref{eq-no-traing:app-first-term-last-constraint})]\label{lem:SDR-relaxation}
	Define the tightened SDP relaxation of (\ref{eq-no-triang:app-first-term-objective}--\ref{eq-no-traing:app-first-term-last-constraint}) as the following semidefinite optimization problem (SDP) with symmetric matrix variable $A=A^\top\in{\mathbb R}^{V^c\times V^c}$ and the abbreviation $\alpha:=A\en{V_c}$: 
	\begin{align}
	(B^{c\rightarrow c''}_\text{SDR-tightened})^2\quad:=\\
	~\text{maximize}_A~~~~&\mathrm{Tr}[\Kout^{VV}A]-2\alpha^\top \Kout^{VM}\frac{\en{n_M}}{n_M}+ \frac{\en{n_M}^\top}{n_M} \Kout^{MM}\frac{\en{n_M}}{n_M}\label{eq:tightened-SDR-objective}\\
	\text{subject to}~~~~&\mathrm{Tr}[\Kin^{VV}A]-2\alpha^\top \Kin^{V}\frac{\en{n_M}}{n_M}+\frac{\en{n_M}^\top}{n_M} \Kin^{MM}\frac{\en{n_M}}{n_M}\leq (B^{c'\rightarrow c})^2 ~~~\forall c'<c, \label{eq:tightened-SDR-quadratic-constraint}\\
	&\left(\begin{array}{cc}A&\alpha\\\alpha^\top&1\end{array}\right)\geq0~~\text{(i.e.\ the left-hand-side is a positive semidefinite matrix)},\label{eq:tightened-SDR-relaxation-A-alpha}\\
	&\en{V^c}^\top A\en{V^c}=1,\label{eq:tightened-SDR-matrix-equality-constraint}\\
	&A\succcurlyeq0~~\text{(entry-wise; above the diagonal suffices due to constraint (\ref{eq:tightened-SDR-relaxation-A-alpha}) and $A=A^\top$)}.\label{eq:tightened-SDR-last-tightening-constraint}
	\end{align}
	Then its optimal value satisfies $(B^{c\rightarrow c''}_\text{SDR-tightened})^2\geq(B^{c\rightarrow c''})^2$, i.e.\ $(B^{c\rightarrow c''}_\text{SDR-tightened})^2$ is an upper bound on the optimum of the non-convex problem (\ref{eq-no-triang:app-first-term-objective}--\ref{eq-no-traing:app-first-term-last-constraint}), which itself is (the square of) an upper bound on the (unknown) MMD discrepancy $D(p|_{c\to c''},q|_{c\to c''})$ (cf.\ Eq.\ (\ref{max-discrepancy-objective-alpha}) in Sec.\ \ref{subsec:validation-method}).
\end{lemma}
To solve the relaxed semidefinite programs from Lemma \ref{lem:SDR-relaxation}, we use the library \texttt{CVXPY} \citep{diamond2016cvxpy}. 
Note that due to the relaxation, the value $(B^{c\rightarrow c''}_\text{SDR-tightened})^2$ of the relaxed optimization from Lemma \ref{lem:SDR-relaxation} is in general different (larger, i.e.\ worse) than the $(B^{c\rightarrow c''})^2$ from the original optimization (\ref{eq-no-triang:app-first-term-objective}); 
this can happen if the found optimum $\widehat{A}$ of the relaxed problem in Lemma\ \ref{lem:SDR-relaxation} cannot be expressed as $\widehat{A}=\widehat{\alpha}\widehat{\alpha}^\top$ (i.e.\ the optimal $\widehat{A}$ is not of rank 1). 
To detect such a \emph{relaxation gap}, one can plug the found vector $\widehat{\alpha} :=\widehat{A}\en{V_c}$ into (\ref{eq-no-triang:app-first-term-objective}) and compare its value $\text{Opt}_{\text{orig}}(\widehat{\alpha})$ to the optimal value $\text{Opt}_{\text{relax}}(\widehat{A})=(B^{c\rightarrow c''}_\text{SDR-tightened})^2$ of (\ref{eq:tightened-SDR-objective}). 
Then, by Lemma \ref{lem:SDR-relaxation} and the maximization (\ref{eq-no-triang:app-first-term-objective}), it holds that the true optimum $(B^{c\rightarrow c''})^2$ of the nonconvex problem (\ref{eq-no-triang:app-first-term-objective}) is sandwiched between two convexly computable quantities:
\begin{align}
\text{Opt}_{\text{orig}}(\widehat{\alpha})~\leq~(B^{c\rightarrow c''})^2~\leq~\text{Opt}_{\text{relax}}(\widehat{A})=(B^{c\rightarrow c''}_\text{SDR-tightened})^2.
\end{align}
When both values agree, $\text{Opt}_{\text{orig}}(\widehat{\alpha})=\text{Opt}_{\text{relax}}(\widehat{A})$, then the relaxation was \emph{tight}, i.e.\ we have certified that $(B^{c\rightarrow c''})^2=(B^{c\rightarrow c''}_\text{SDR-tightened})^2=\text{Opt}_{\text{relax}}(\widehat{A})$ are equal and our relaxation has found the true optimum. 
More generally, we can guarantee a \emph{relaxation gap} $\Delta=(B^{c\rightarrow c''}_\text{SDR-tightened})^2-(B^{c\rightarrow c''})^2$ of at most $\widehat{\Delta}=\text{Opt}_{\text{relax}}(\widehat{A})-\text{Opt}_{\text{orig}}(\widehat{\alpha})$, and an approximation ratio $\gamma=(B^{c\rightarrow c''})^2/(B^{c\rightarrow c''}_\text{SDR-tightened})^2\in[0,1]$ of at least $\widehat{\gamma}=\text{Opt}_{\text{orig}}(\widehat{\alpha})/\text{Opt}_{\text{relax}}(\widehat{A})\in[0,1]$.

Empirically we find in this way that the relaxation in Lemma \ref{lem:SDR-relaxation} is basically tight in most problem instances, i.e.\ it returns (almost) the correct optimum of (\ref{eq-no-triang:app-first-term-objective}--\ref{eq-no-traing:app-first-term-last-constraint}). See App.\ \ref{sec:empirical-tightness} for an empirical evaluation.

With this relaxation, the number of optimization variables increases from $\text{dim}(\alpha)=V^c$ in (\ref{eq-no-triang:app-first-term-objective}--\ref{eq-no-traing:app-first-term-last-constraint}) to $\text{dim}(A=A^\top)=V^c(V^c+1)/2\sim(V^c)^2/2$ in (\ref{eq:tightened-SDR-objective}--\ref{eq:tightened-SDR-last-tightening-constraint}), i.e.\ it grows quadratically with the number of validation points $V^c$ for component $c$; the number of optimization constraints also increases like $\sim C+(V^c)^2$. This results in a computational limitation which restricts the number of validation inputs to roughly $V^c\lesssim10^3$ with standard convex solvers on standard computing hardware. 
If one would like to apply our general validation method with more validation points $V^c$, one would have to to find another way to efficiently upper-bound the optimization problem (\ref{eq-no-triang:app-first-term-objective}--\ref{eq-no-traing:app-first-term-last-constraint}), instead of our tightened SDP relaxation (Lemma \ref{lem:SDR-relaxation}).

To run the whole validation method in Algorithm \ref{algorithm:DPBound}, the number of bound optimizations (\ref{eq-no-triang:app-first-term-objective}--\ref{eq-no-traing:app-first-term-last-constraint}) or (\ref{eq:tightened-SDR-objective}--\ref{eq:tightened-SDR-last-tightening-constraint}) to perform equals the number of connections $c\to c''$ between components of the system (where $1\leq c<c''\leq C+1$, see Sec.\ \ref{sec:setup}), i.e.\ the number of such bound computation lies between $C$ (for the linear chain) and $C(C+1)/2$ (for the ``fully connected'' system). See App.\ \ref{sec:runtime} for actual runtimes of our method.

\section{DETAILS ON THE FAILURE PROBABILITY OPTIMIZATION IN EQ.\ (\ref{max-failure-objective-alpha})}\label{app:violation-optimization}
Similar to (\ref{eq-no-triang:app-first-term-objective}--\ref{eq-no-traing:app-first-term-last-constraint}), we can formulate a sample-based optimization for the optimization problem in Eq.\ (\ref{max-failure-objective-alpha}) in Sec.\ \ref{subsec:validation-method} in the case of using MMD as the discrepancy measure, to compute an upper bound $F_\text{max}$ on the system failure probability $p_\text{fail}$. 
More precisely, assuming that this MMD measure on the TPI output $y^C$ has kernel $k^{y}\equiv k^{C\to C+1}$, we can write this problem (optionally with the monotonicity and Lipschitz conditions mentioned below Eq.\ (\ref{max-failure-objective-alpha}) in Sec.\ \ref{subsec:validation-method} with the optimization variable $\alpha\in{\mathbb R}^V$:

\newcommand{\Kyy}{\ensuremath{{\bf{K}}_y}}
\begin{align}
F_{\text{max}} & = \text{maximize}_\alpha  \sum_{v: g_v>\tau} \alpha_v\label{eq:violation_objective}\\
& \qquad \text{subject to }~ \alpha ^\top \Kyy^{VV}\alpha  - 2\alpha^\top \Kyy^{VM}\frac{\en{n_M}}{n_M} + \frac{\en{n_M}^\top}{n_M} \Kyy^{MM}\frac{\en{n_M}}{n_M} \leq (B^y)^2, \label{eq:violation_mmd_constraint}\\
& \qquad\qquad\qquad~~\alpha\geq0, \qquad\en{V}^\top\alpha =1,\\
&\qquad\qquad\qquad~~\alpha_v \leq \alpha_{v-1}\quad \forall v\text{  with  }g_v\geq \tau'\quad\text{(monotonicity; we take $\tau':=\tau$)},\label{eq:violation_monotonicity_constraint}\\
& \qquad\qquad\qquad~~|\alpha_{v+1}-\alpha_v|\leq\Lambda_\text{max}|g_{v+1}-g_v|\quad \forall v\quad\text{(Lipschitz condition)},\label{eq:violation_lipschitz_constraint}
\end{align}
where we defined kernel matrices on the TPI grid-points $g_v$ and simulation outputs $y^M_n$:
\begin{align}
&\qquad\qquad (\Kyy^{VV})_{v,v'}  = k^y\left(g_v, g_{v'}\right),\nonumber\\
&\qquad\qquad (\Kyy^{VM})_{v,n'}  = k^y\left(g_v, y^M_{n'}\right),\nonumber\\
& \qquad\qquad(\Kyy^{MM})_{n,n'}  = k^y\left(y^M_n, y^M_{n'}\right).\nonumber
\end{align}

The optimization problem (\ref{eq:violation_objective}--\ref{eq:violation_lipschitz_constraint}) has a linear objective and linear constraints except for the quadratic MMD constraint (\ref{eq:violation_mmd_constraint}). 
It is thus a convex optimization problem that can be solved efficiently and exactly, without relaxations (unlike required for the bound optimization (\ref{max-discrepancy-objective-alpha}) in Sec.\ \ref{subsec:validation-method}, see App.\ \ref{app:semidefinite-relaxation}).

The above formulation depends on a set of grid-points $g_v$ and a Lipschitz constant $\Lambda_{\text{max}}$. Ideally, the $\Lambda_\text{max}$ should be a tight upper bound on the Lipschitz constant of the system's TPI output density $p_y$. 
As we do not know this density $p_y$, we use a heuristic estimator of $\Lambda_\text{max}$ computed from histograms of the simulation output distribution $q_y$ as a proxy (see experimental details in App.\ \ref{sec:deta-exper-sect}). 
In order for the solution of (\ref{eq:violation_objective}--\ref{eq:violation_lipschitz_constraint}) to be close to its true value (\ref{max-failure-objective-alpha}) in Sec.\ \ref{subsec:validation-method}, the grid-points need to be sufficiently dense in order to reveal differences as measured by the MMD constraint (\ref{eq:violation_mmd_constraint}) (see also App.~\ref{app:proof-proposition}). 
As the MMD measure $D(p_y,q_y)$ with kernel $k^y$ can also be interpreted as the L2-distance between the corresponding kernel density estimators of the samples from $p_y$ and $q_y$ \citep{gretton2012kernel}, we choose the grid-spacing relative to the lengthscale $\ell$ of the kernel $k^y$; more precisely, in our experiments we require that $g_{v+1}-g_v\leq \frac{\ell}{5}$. 
Additionally, the grid-points should cover the range where the support of both $p_y$ and $q_y$ lies (although $p_y$ is not known). 
For our experiments, we chose the grid range $[g_1,g_V]$ such that it contains all data-points from $q_y$, as well as a significantly large region around the threshold $\tau$ (see App.\ \ref{sec:deta-exper-sect}). 
Even though the system's TPI output distribution $p_y$ is not known, in many applications the plausible (or even the potential) range of TPI values $y^C$ will typically be known from domain knowledge; 
in this case, the grid endpoints should be chosen to cover this range.

Note, that the optimization for the failure probability in Eq.\ (\ref{max-failure-objective-alpha}) in Sec.\ \ref{subsec:validation-method} or in Eqs.\ (\ref{eq:violation_objective}--\ref{eq:violation_lipschitz_constraint}) is also possible for higher-dimensional TPI quantities $y^C$. 
In this case, a specification $y^C\in \mathcal{TPI}_\text{fail}$ of the critical region is required (replacing the specification $y^C>\tau$ in (\ref{eq:violation_objective})).  
Even if this specification is non-linear, the optimization objective will remain linear and the constraints quadratic, again by choosing a grid on the TPI space as in the one-dimensional case.

By construction, the final solution $F_{\text{max}}$ (computed via (\ref{max-discrepancy-objective-alpha}),(\ref{max-failure-objective-alpha} in Sec.\ \ref{subsec:validation-method}), or more concretely via (\ref{eq:tightened-SDR-objective}--\ref{eq:tightened-SDR-last-tightening-constraint}),(\ref{eq:violation_objective}--\ref{eq:violation_lipschitz_constraint})) is an upper bound on the system's true failure probability $F_{\text{max}} \geq p_{\text{fail}}=\int \mathds{1}_{y>\tau}dp_y(y)=\int \mathds{1}_{S(x)>\tau}dS(x)dp_x(x)$  (see Prop.\ \ref{prop:convergence} in Sec.\ \ref{subsec:validation-method} and also App.\ \ref{app:proof-proposition}); 
as a result it can be used for virtual system validation. 
This bound $F_{\text{max}}$ remains valid for any discrepancy measure, choice of kernels or lengthscales. 
For example, when choosing MMD with a very large kernel lengthscale as discrepancy measure, its discriminative power is minimal, resulting in very small discrepancy values and hence small corresponding bounds $B^{c\to c''}$. 
However, for such a measure, it is also difficult to distinguish $p_\alpha$ from the given $q|_{c\to c''}$ in a later bound propagation step, counteracting the small obtained $B^{c\to c''}$ and potentially resulting in a larger final $F_{\text{max}}$. 
We, therefore, use the final $F_{\text{max}}$ as the minimization objective in a Bayesian Optimization scheme \citep{frohlich2020noisy} to select the kernel parameters, see also App.\ \ref{sec:deta-exper-sect}. Another option to select good kernel parameters would be gradient-based minimization of $F_{\text{max}}$ w.r.t.\ to the kernel parameters, which is possible in our framework as the convex programs from Apps.\ \ref{app:semidefinite-relaxation} and \ref{app:violation-optimization} can be (automatically) differentiated \citep{agrawal2019differentiable}.

\section{PROOF AND EXTENSIONS OF PROP.\ \ref{prop:convergence} (SEC.\ \ref{subsec:validation-method})}\label{app:proof-proposition}
\begin{proof}[Proof of Prop.\ \ref{prop:convergence}]
	Under the assumptions {\it{(i)}} and {\it{(ii)}} that the set of validation inputs $\{x^c_v\}_v$ contains  \emph{all} actually occurring input points into $S^c$ and that $p_\alpha=\sum_v\alpha_v\delta_{x^c_v}S^c(x^c_v)$ is built with the \emph{correct} outputs $S^c(x^c_v)$, this set of distributions $p_\alpha$ over which we optimize in (\ref{max-discrepancy-objective-alpha}) (see Sec.\ \ref{subsec:validation-method}) contains the joint real-world distribution of in- and outputs of $S^c$ (i.e.\ the marginal of the real-world distribution $p$ capturing the joint in- and outputs of $S^c$). Thus, if the input bound values $B^{c'\to c}$ were true upper bounds on the actual $D(p|_{c'\to c},q|_{c'\to c})$, then $B^{c\to c''}$ from (\ref{max-discrepancy-objective-alpha}) in Sec.\ \ref{subsec:validation-method} is also a true upper bound on the actual $D(p|_{c\to c''},q|_{c\to c''})$. 
	By induction on $c=1,2,\ldots,C$ we can thus conclude that $B^y\equiv B^{C\to C+1}$ is a true upper bound on the real-world discrepancy $D(p_y,q_y)\equiv D(p|_{C\to C+1},q|_{C\to C+1})$ if only the initial bound values $B^{0\to c}$ were true upper bounds on the actual initial discrepancies $D(p|_{0\to c},q|_{0\to c})$; 
	we assume this last statement about the initial bound values $B^{0\to c}$ to be true since they are supposed to be given in that way (alternatively, the same statement can be concluded with high confidence $\geq1-\delta$ if the $B^{0\to c}$ were computed via samples from $p_x$ \citep{gretton2012kernel}; see end of Sec.\ \ref{sec:setup}). 
	Finally, by the same reasoning, under the assumption {\it{(iii)}} that the set $\{g_v\}_v$ of grid-points contains \emph{all} occurring real-world TPI values, the optimization (\ref{max-failure-objective-alpha}) from Sec.\ \ref{subsec:validation-method} translates the true upper bound $B^y$ into a true upper bound $F_\text{max}$ on the real-world failure probability $p_\text{fail}$.
\end{proof}

Note that further upper-bounding the optimizations (\ref{max-discrepancy-objective-alpha}),(\ref{max-failure-objective-alpha}) from Sec.\ \ref{subsec:validation-method} by relaxations as in App.\ \ref{app:semidefinite-relaxation} leads to valid upper bounds $F_\text{max}$ as well, by the same reasoning as in the above proof.

The assumptions {\it{(i)}} and {\it{(ii)}} of Prop.\ \ref{prop:convergence} are so strong that one basically knows all system maps $S^c$ \emph{explicitly}, at least on all those inputs points that occur in the real world. 
If one would, in addition, know the real-world input distribution $p_x$ (e.g.\ in a sample-based way), one could (in theory) simulate the system map $S$ on all those samples and compute (or at least estimate) the real-world TPI distribution $p_y$ by Monte-Carlo sampling; thus determine the desired $p_\text{fail}$ arbitrarily well. 
However, we do \emph{not} need the input distribution $p_x$ to be known explicitly for our proposed method to be applicable; and we apply our method even when the strong knowledge about the $S^c$ implied by {\it{(i)}} and {\it{(ii)}} is \emph{not} available.

\begin{remark}[Upper bounds in the limit]
	Beyond the strong assumptions of Prop.\ \ref{prop:convergence}, the upper bounds obtained by our method (\ref{max-discrepancy-objective-alpha}),(\ref{max-failure-objective-alpha}) (see Sec.\ \ref{subsec:validation-method}) can be proven to be valid under weaker, more realistic assumptions. 
	This appears possible, for example, in the following scenario, as the numbers $V^c\equiv V$ of available validation data points grow: 
	\emph{(a)} The set of validation inputs $\{x^c_v\}_{v=1}^V$ covers the input space $\mathcal{S}^c_{in}\subset{\mathbb R}^{d^c_{in}}$ of $S^c$ increasingly densely as $V\to \infty$, e.g.\ in the sense that $\max_{x\in{\mathcal S}^c_{in}}\min_{v\in\{1,\ldots,V\}}\|x-x^c_v\|\ \longrightarrow\ 0\ \text{as}\ V\to\infty$; with an analogous condition for the set of grid-points $\{g_v\}_{v=1}^V$ used in (\ref{max-failure-objective-alpha}) from Sec.\ \ref{subsec:validation-method}. 
	\emph{(b)} As discrepancy measures $D^{c'\to c}$ we use MMD distances w.r.t.\ continuous and bounded kernel functions $k^{c'\to c}$. This is satisfied by all kernels used here, such as the squared-exponential and IMQ kernels, even when applied after a data embedding \citep{gretton2012kernel}.
	\emph{(c)} The real-world subsystem maps $S^c:x^c\mapsto S^c(x^c)$ (which are not known explicitly, and whose output is a probability distribution in general) are continuous w.r.t.\ the discrepancy measures $D^{c\to c''}$ at their output. This condition would be implied by continuity w.r.t.\ the Wasserstein or the total variation distances and might in some cases be argued from physical considerations. For deterministic components $S^c$, this requirement simply means that the deterministic mapping is continuous. 
	\emph{(d)} For each validation input $x^c_v$ we have measured a sufficient number $W$ of i.i.d.\ samples $y^c_{v,w}\sim S^c(x^c_v)$ ($w=1,\ldots,W$) from the \emph{true but unknown} output distribution $S^c(x^c_v)$ and we use the sample-based $p_\alpha:=\sum_{v=1}^V\sum_{w=1}^W\frac{\alpha_v}{W}\delta_{x^c_v}\delta_{y^c_{v,w}}$ in the optimization (\ref{max-discrepancy-objective-alpha}) from Sec.\ \ref{subsec:validation-method}. We need $W=W(V)$ to be large enough that the empirical estimate $(1/W)\sum_{w=1}^W\delta_{y^c_{v,w}}$ is sufficiently close to $S^c(x^c_v)$ in the kernel mean embedding \citep{muandet2017KME}, instead of the exact equality required by Prop.\ \ref{prop:convergence}. 
	\emph{(e)} The initial $B^{0\to c}$ must either be true upper bounds on the actual $D(p|_{0\to c},q|_{0\to c})$, or must be computed empirically from an increasing number of i.i.d.\ samples $\{x_v\}_{v=1}^V$ coming from the real-world distribution $p_x$ \citep{gretton2012kernel} (see end of Sec.\ \ref{sec:setup}).
	
	Even under such more realistic circumstances, one may still obtain a  provably valid upper bound $F_\text{max}$ on the real-world failure probability $p_\text{fail}$ via our method in the limit $V\to\infty$ of sufficiently many validation points, at least almost surely over the sampling of $y^c_{v,w}$ (and $x_v$ in the case where $B^{0\to c}$ are estimated from those samples) and up to any additive $\varepsilon>0$ chosen beforehand. 
	This can be shown by following the steps of the proof of Prop.\ \ref{prop:convergence}, replacing each valid upper bound or exact equality by an approximation or limiting argument, using the above assumptions. 
	Even more, when e.g.\ the kernels $k^{c'\to c}$ as well as Lipschitz constants of the maps $S^c$ are known, then more effective statements can be obtained, in the sense that $V$ can then be related to $\varepsilon$ and to the confidence in the final $F_\text{max}$ being a valid bound. 
	
	However, even such more realistic convergence statements would still not be practical in all cases, since e.g.\ the assumption \emph{(a)} generally requires a number of validation inputs $V\gtrsim (1/\delta)^{d^c_{in}}$ \emph{exponential} in the dimensions of the input spaces of the $S^c$ (where $\delta= \max_{x\in{\mathcal S}^c_{in}}\min_{v\in\{1,\ldots,V\}}\|x-x^c_v\|$ denotes the desired set approximation accuracy). 
	We, therefore, take a pragmatic viewpoint, in that we apply our method (\ref{max-discrepancy-objective-alpha}),(\ref{max-failure-objective-alpha}) (Sec.\ \ref{subsec:validation-method}; see also Algorithm \ref{algorithm:DPBound}) even in those cases where we only have a limited amount of validation data available. We evalute this empirically in Sec.\ \ref{sec:reli-benchm-eval}. 
\end{remark}

\begin{remark}[Arbitrary simulation]
	For the justification of our method -- either heuristically or rigorously as above -- it is \emph{not} necessary that the ``simulation distribution'' $q$ be in any sense close to the true system behavior or that the simulations $M^c$ need to be faithful approximations of the actual system components $S^c$. 
	Rather, it suffices that each of the distributions $q|_{\tilde{c}\to\hat{c}}$ (i.e.\ one distribution for each pair $(\tilde{c},\hat{c})$ with $0\leq\tilde{c}<\hat{c}\leq C+1$) has the same value in each of the optimizations (\ref{max-discrepancy-objective-alpha}) and (\ref{max-failure-objective-alpha}) (Sec.\ \ref{subsec:validation-method}); 
	each of these distributions $q|_{\tilde{c}\to\hat{c}}$ simply acts as an ``anchor'' with respect to which the (unknown) real world distribution $p|_{\tilde{c}\to\hat{c}}$ is assessed -- and these anchors need to remain fixed. 
	This means that e.g.\ the full joint distribution $q$ could have been generated by starting from an arbitrarily chosen $q_x$ and arbitrarily ``bad'' models $M^c$, and that the resulting $F_\text{max}$ should remain an upper bound on $p_\text{fail}$ regardless.
	In practice we nevertheless desire the simulations $M^c$ to be close to $S^c$ as we intuitively believe that such closeness should lead to stronger statements about the system via the simulations, i.e.\ that the upper bound $F_ \text{max}$ be as good (i.e.\ small or tight) as possible. 
	We investigate this dependency in the experiments (Sec.\ \ref{sec:experiments}) via the comparison of Perfect Model vs.\ Misfit Model.
\end{remark}

\section{DETAILS ON EXPERIMENTS (SEC.\ \ref{sec:experiments})}
\label{sec:deta-exper-sect}

This section will provide all experiment details to reproduce the result shown in Sec.~\ref{sec:reli-benchm-eval}.
It will contain the description of the reliability benchmark problems, the (input) data generation process, the per-channel/signal kernels and their length scales, the construction of the bias input (i.e. Biased Input), the misfit model (i.e. Misfit Model) and hyper-parameters of the failure probability computation.
Most settings are kept fix across all benchmark problems (described in Sec.~\ref{sec:fix_settings}) and only the biased input construction and kernel parameters vary for each problem (described in Sec.~\ref{sec:per_use_case_settings}).

\subsection{Fixed Settings}
\label{sec:fix_settings}

Across all experiments, the following settings are kept fixed in order to be able to compare the validation performances.

\textbf{Data}:
For all subsystems $S^c$ and sub-models $M^c$ ($c=1,2,...,C$) we fix the number of samples to $V_c=100$ and $n_M=500$, respectively.
All data is generated by a fixed base seed of $2349$.\\
\textbf{Trials}:
All experiments report the mean and standard deviation across $5$ independent trials where all settings are kept identical except the base seed; this is incremented by the 0-index id of each trial.\\
\textbf{Failure probability threshold}:
The threshold for each validation problem is set such that the ground truth failure probability of the (TPI) output of the last subsystem $S^C$ is at approximately $1\%$; $1$ million samples are used to determine the threshold.
The goal of all validation methods across all benchmark problems is to achieve a failure probability as close to $1.0\%$ as possible.
It should be noted that in a real-world system, the ground truth failure probability is not available or is poorly approximated through limited samples.\\
\textbf{Failure probability grid}:
The grid of the failure probability optimization is based on the (TPI) output of the last sub-model $M^C$.
It is chosen to cover the entire support of the model output distribution and a significantly large grid region after the threshold $\tau$. \\
\textbf{Failure probability Lipschitz constant}:
Ideally, the Lipschitz-constant should be set according to the smoothness (or Lipschitz constant) of the system's TPI distribution.
As this is not available, we used the empirical histogram distribution of simulated TPI outputs as a proxy.
More precisely, for each problem a histogram with $100$ bins is constructed to compute the resulting empirical Lipschitz-constant. \\
\textbf{Model misfit -- Gaussian process}:
For the cases where we artificially introduce a model misfit in sub-models $M^c$ ($c=1,2,...,C$) (i.e. ``Model Misfit'' in main Tab.\ \ref{tab:results_benchmark}) we learn a Gaussian process (GP) for each sub-model $M^c$ to model the input-output mapping of the corresponding subsystem $S^c$.
The goal is to introduce a misfit in a controllable fashion which reflects realistic misfits between real-world systems and simulations.
We generate $100$ independent training samples for learning the GP (i.e. the training data is not used to solve the validation problem).
We use exact GP inference in \texttt{GPyTorch} \citep{gardner2018gpytorch} with Radial basis function (RBF) kernels.
The \texttt{LBFGS-scipy} optimizer is used with a learning rate of $1.0\mathrm{e}{-3}$ for $2000$ iterations and $10$ restarts with different initializations.
All misfit models use the exact same GP settings, therefore the amount of misfit varies across benchmark problems and their individual components. \\
\textbf{Biased input}:
In order to test the limits of the system and validation methods, we specifically use a biased (model) input $q_{\text{biased}-x}$ for each benchmark problem.
This bias was constructed such that the TPI model output distribution is shifted farther away from the threshold.
Such a bias distribution exploits the weakness of some validation methods (e.g. MCCP and SurrModel) which do not take the input discrepancy between $p_x$ and $q_x$ into account.
As a result, using the bias input distribution $q_x$ for the model input yields overly optimistic estimates with a severe underestimation the failure probability (i.e. $F_{\text{max}} < F_{\text{GT}} = 1.0$). \\
\textbf{Kernels}:
All experiments either use a radial basis function (RBF) (squared-exponential) kernel or inverse multiquadratic (IMQ) (also known as rational quadratic) kernel~\citep{gorham2017measuring}.
Both kernels use a jitter of $1\mathrm{e}{-10}$, and the IMQ kernel has a fixed $\alpha=-0.5$.
The length scales for both kernels differ for each benchmark problem and channel/signal dimension. \\
\textbf{Length scale search}:
For a given benchmark problem, all length scales across all channels (including each dimension) are optimized to minimize the failure probability $F_{\text{max}}$.
It should be note that as $F_{\text{max}}$ is a probability, its interpretation is independent of the kernel length scales.
As a result, one can perform a kernel length scale search with the objective of minimizing $F_{\text{max}}$.
For all problems, we performed a Bayesian optimization search~\citep{frohlich2020noisy} where the search space of all parameters are kept large ($[1.0\mathrm{e}{-8}, 5.0\mathrm{e}{3}]$), except the length scale of the last TPI kernel which depends on the output range of each problem.
Furthermore, we only perform the length scale search for the setting "Perfect Input" and "Perfect Model" (i.e. top left quadrant in main Tab.\ \ref{tab:results_benchmark}).
The found length scales are kept fixed for all other settings.

\subsection{Reliability Benchmark Problems}
\label{sec:per_use_case_settings}
The following subsections provide all details for each reliability benchmark problem/dataset used in main Tab.\ \ref{tab:results_benchmark}. \\

\subsubsection{Controlled Solvers~\citep{sanson2019systems}}
\textbf{Components}:
This problem has $4$ components (i.e. solvers): the Sobol function, Ishigami function, and the remaining two are products of polynomial functions and trigonometric functions.
\begin{align}
f_1 &:(x_{1:5}) \mapsto \prod_{k=1}^5 \frac{|4x_k - 2| + a_k}{1 + a_k} =:x_6 \\
f_2  &:(x_{6:8}) \mapsto \sin{x_6} + 0.7\sin^2{x_7} + 0.1x_8^4 \sin{x_6} =: x_9\\
f_3  &:(x_{9:14}) \mapsto x_{10}^2 \arctan{1-x_{14}} + x_{11} x_{12} x_{13}^3 + 3x_9 =:x_{15}\\
f_4  &: (x_{15:19})\mapsto \sin{x_{19}}x_{18} + x_{15}x_{16} + x_{17},
\end{align}
where $a=(12, 2, 3, 4, 45)$,
the $x_6 = f_1(x_{1:5})$, $x_9 = f_2(x_{6:8})$, and $x_{15} =  f_3(x_{9:14})$.
This problem is also defined in more detail in Sec.\ 5.5. ``Test Case 3'' in \citet{sanson2019systems} with Fig.\ 14 depicting the causal graph. \\
\textbf{Perfect Input}:
$16$-dimensional input ($x_{\{1,2,...,19\} \setminus \{6, 9, 15\}}$) sampled from $\mathcal{U}_{[0.0, 1.0]}$. \\
\textbf{Biased Input}:
Input signal $x_{18}$ sampled from $\mathcal{U}_{[0.0, 0.8]}$ and the remaining inputs ($x_{\{1,2,...,19\} \setminus \{6, 9, 15\}}$) from $\mathcal{U}_{[0.0, 1.0]}$. \\
\textbf{Model Misfit}:
A GP is learned for each component $M^c$ ($c=1,2,...,C$). \\
\textbf{Failure probability}:
The grid range is fixed to $[g_{\text{min}}=-5.0, g_{\text{max}}=60.0]$ with the threshold at $\tau=14.51$.
A (decreasing) monotonicity constraint is enforced for the range $[\tau - 1.5*\text{grid-spacing}, g_{\text{max}}$.
The Lipschitz constant is set to $0.28$. \\
\textbf{Kernels}:
All channels use a RBF kernel with the following length scales: $[1.0\mathrm{e}{-6}, 5.0e01, 1.0\mathrm{e}{-6}, 5.0\mathrm{e}{1}, 1.0\mathrm{e}{-6}, 1.0\mathrm{e}{-6}, 1.0\mathrm{e}{-6}, 6.397]$ for the channels $[x_{1:5}, x_{6:7}, f_1(\cdot), x_{8:12}, f_2(\cdot), x_{13:16}, f_3(\cdot), f_4(\cdot)]$, respectively.

\subsubsection{Chained Solvers~\citep{sanson2019systems}}
\textbf{Components}:
This problem has $2$ components (i.e. solvers) forming a composition of two univariate functions $f_1$ and $f_2$:
\begin{align}
f_1 &: x \mapsto e^{\sqrt{x}}\sin{x} + 6 e^{-(x-2)^2} + \frac{5}{2} e^{-3(x-1)^2} \\
f_2 &: x \mapsto \sin{x} + 0.3  x  \sin{3.4x + 0.5},
\end{align}
where the global output is $f = f_1 \circ f_2$.
This problem is also defined in more detail in Sec.\ ``5.1. Test Case 1'' in \citep{sanson2019systems} with Fig.\ 4 plotting the signals. \\
\textbf{Perfect Input}:
Univariate input sampled from $\mathcal{U}_{[0.0, 6.0]}$. \\
\textbf{Biased Input}:
Univariate input sampled from a mixture distribution $\alpha \mathcal{U}_{[0.0, 6.0]} + (1-\alpha) \mathcal{U}_{[4.0, 6.0]}$ with $\alpha=0.90$.
The $\alpha$ controls the trade-off between the biasedness and correctness of the support of the resulting distribution. \\
\textbf{Model Misfit}:
A GP is learned for each component $M^c$ ($c=1,2,...,C$). \\
\textbf{Failure probability}:
The grid range is fixed to $[g_{\text{min}}=-8.0, g_{\text{max}}=5.0]$ with the threshold at $\tau=1.459$.
A (decreasing) monotonicity constraint is enforced for the range $[\tau - 1.5*\text{grid-spacing}, g_{\text{max}}$.
The Lipschitz constant is set to $99.0$. \\
\textbf{Kernels}:
The $2$ single-dimensional input channels and the TPI channel use a RBF and IMQ kernel, respectively, with length scales $[1\mathrm{e}{-8}, 1\mathrm{e}{-8}, 1.218]$.

\subsubsection{Borehole~\citep{sim_bench_website}}
\textbf{Components}:

This problem has a single component which models water flow through a borehole:
\begin{align}
f: (r_w, r, T_u, H_u, T_l, H_l, L, K_w) \mapsto \frac{2\pi T_i(H_u-H_l)}{\text{ln}(r/r_w)(1+\frac{2LT_u}{\text{ln}(r/r_w)r_w^2K_w} + \frac{T_u}{T_l})} 
\end{align}
This problem is also defined in more detail at \url{https://www.sfu.ca/~ssurjano/borehole.html}.
We construct two variants of this problem: \texttt{single\_borehole} and \texttt{compositional\_borehole}.
The former considers only the function $f$ above, whereas the latter breaks the function up into multiple (5) smaller components:
\begin{align}
f_1 &: (T_u, H_u, H_l) \mapsto 2\pi T_i(H_u-H_l) \\ 
f_2 &: (r_w, r, T_u, L, K_w) \mapsto (\frac{2LT_u}{\text{ln}(r/r_w)r_w^2K_w} \\ 
f_3 &: (T_u, T_l) \mapsto \frac{T_u}{T_l} \\
f_4 &: (r_w, r, y2, y3) \mapsto \text{ln}(\frac{r}{r_w} (1+y_2+y_3))  \\
f_5 &: (y_1, y_4) \mapsto \frac{y_1}{y_2},
\end{align}
where $y_1=f_1(\cdot)$, $y_2=f_2(\cdot)$, $y_3 = f_3(\cdot)$, and $y_4 = f_4(\cdot)$.\\ 
\textbf{Perfect Input}:
The $8$ inputs are sampled from distributions described in detail at \url{https://www.sfu.ca/~ssurjano/borehole.html}.
The description also includes the range of all signals in the system. \\
\textbf{Biased Input}:
The sampling range of the $H_u$ signal is modified from $[990, 1110]$ to $[990, 1010]$. \\
\textbf{Model Misfit}:
A GP is learned for each component $M^c$ ($c=1,2,...,C$). \\
\textbf{Failure probability}:
The grid range is fixed to $[g_{\text{min}}=-35.0, g_{\text{max}}=600.0]$ with the threshold at $\tau=157.1$.
A (decreasing) monotonicity constraint is enforced for the range $[\tau - 1.5*\text{grid-spacing}, g_{\text{max}}$.
The Lipschitz constant is set to $0.0006$. \\
\textbf{Kernels}:
For the \texttt{single\_borehole}, the RBF kernel is used for both input and output kernels with the following length scales: $[[10.599, 6.587, 24.609, 32.369, 46.431, 23.046, 12.943, 2.734], 23.578]$, respectively.
The first kernel has a multi-dimensional length scale; one for each input dimension.
For the \texttt{compositional\_borehole}, the RBF kernel is used for all $8$ input channels and a IMQ kernel for the TPI output kernel with length scales:
$[5.0\mathrm{e}{+3}, 1.\mathrm{e}{-1}, 3.198\mathrm{e}{+3}, 5.0\mathrm{e}{+3}, 1.0\mathrm{e}{-1}, 5.0\mathrm{e}{+3}, 5.0\mathrm{e}{+3}, 1.0\mathrm{e}{-1}, 5.0\mathrm{e}{+3}, 2.634]$, respectively.

\subsubsection{Branin~\citep{sim_bench_website}}
\textbf{Components}:
This problem has a single component:
\begin{align}
f: x_{1:2} \mapsto f_{\text{max}} - a (x_2 -bx_1^2 + cx_1 - r)^2 + s(1-t)\cos{x_1} + s,
\end{align}
where we the recommended parameters are used: $a = 1$, $b = 5.1/(4\pi^2)$, $c = 5 / \pi$, $r = 6$, $s = 10$ and $t = 1 / (8\pi)$.
In order to map the minimization problem to our maximization setting, we modify the branin function by adding $f_{\text{max}}=312.0$ and subtracted the original formulation thereof.
This problem is also defined in more detail at \url{https://www.sfu.ca/~ssurjano/branin.html}\\
We construct two variants of this problem: \texttt{single\_branin} and \texttt{compositional\_branin}.
The former considers only the function $f$ above, whereas the latter breaks the function up into multiple (3) smaller components:
\begin{align}
f_1 &: x_{1:2} \mapsto  (x_2 -bx_1^2 + cx_1 - r)^2 \\
f_2 &: x_{1:2} \mapsto  (1-t)\cos{x_1} \\
f_3 &: x_{3:4} \mapsto  f_{\text{max}} - ax_3 + sx_4 + s,
\end{align}
where $x_3=f_1(\cdot)$ and $x_4 = f_2(\cdot)$. \\
\textbf{Perfect Input}:
The $2$ inputs are sampled from $\mathcal{U}_{[-5.0, 10.0]}$ and $\mathcal{U}_{[0.0, 15.0]}$, respectively. \\
\textbf{Biased Input}:
The $2$ inputs are sampled from a mixture distribution $\alpha \mathcal{U}_{[-5.0, 10.0]} + (1-\alpha) \mathcal{U}_{[8.0, 10.0]}$ and $\alpha \mathcal{U}_{[0.0, 15.0]} + (1-\alpha) \mathcal{U}_{[12.0, 15.0]}$ with $\alpha=0.10$, respectively.
The $\alpha$ controls the trade-off between the biasness and correctness of the support of the resulting distribution. \\
\textbf{Model Misfit}:
A GP is learned for each component $M^c$ ($c=1,2,...,C$). \\
\textbf{Failure probability}:
The grid range is fixed to $[g_{\text{min}}=-35.0, g_{\text{max}}=700.0]$ with the threshold at $\tau=330.82$.
A (decreasing) monotonicity constraint is enforced for the range $[\tau - 1.5*\text{grid-spacing}, g_{\text{max}}$.
The Lipschitz constant is set to $0.005$. \\
\textbf{Kernels}:
For the \texttt{single\_branin}, the RBF kernel is used for both input and output kernels with the following length scales: $[0.003, 21.161]$, respectively.
For the \texttt{compositional\_branin}, the IMQ kernel is used for all channels with length scales: $[1\mathrm{e}{-8}, 1\mathrm{e}{-8}, 500.0, 1\mathrm{e}{-8}, 26.064]$.

\subsubsection{Four Branch~\citep{UQworld}}
\textbf{Components}:
This problem has $4$ independent components which form four branches and the final global output takes the minimum of the four component outputs:

\begin{align}
f_1 &: x_{1:2} \mapsto 3 + 0.1(x_1 - x_2)^2 - \frac{x_1+x_2}{\sqrt{2}}\\
f_2 &: x_{1:2} \mapsto 3 + 0.1(x_1 - x_2)^2 + \frac{x_1+x_2}{\sqrt{2}}\\
f_3 &: x_{1:2} \mapsto (x_1 - x_2) + \frac{p}{\sqrt{2}}\\
f_4 &: x_{1:2} \mapsto (x_1 - x_2) - \frac{p}{\sqrt{2}}\\
f_5 &: x_{1:2} \mapsto \text{min} \{f_1(x_{1:2}), f_2(x_{1:2}), f_3(x_{1:2}),f_4(x_{1:2})\} + 10,
\end{align}
where $p=6.0$.
This problem is also defined in more detail in \cite{UQworld}, where Fig.\ 1 shows the surface plot of the four branch function. 
We construct two variants of this problem: \texttt{single\_four\_branch} and \texttt{compositional\_four\_branch}.
The former considers only the function $f_4$ above, whereas the latter breaks the function up into multiple (4) smaller components. \\
\textbf{Perfect Input}:
The $2$ inputs are sampled from two normal distributions described as in detail at \cite{UQworld}. \\
\textbf{Biased Input}:
The $2$ inputs are sampled from a mixture distribution $\alpha \mathcal{N}(0.0, 1.0) + (1-\alpha) \mathcal{N}(0.0, 0.90)$ and $\alpha \mathcal{N}(0.0, 1.0) + (1-\alpha) \mathcal{N}(0.0, 0.9)$ with $\alpha=0.80$, respectively.
The $\alpha$ controls the trade-off between the biasness and correctness of the support of the resulting distribution. \\
\textbf{Model Misfit}:
A GP is learned for each component $M^c$ ($c=1,2,...,C$). \\
\textbf{Failure probability}:
The grid range is fixed to $[g_{\text{min}}=0.0, g_{\text{max}}=30.0]$ with the threshold at $\tau=9.693$.
A (decreasing) monotonicity constraint is enforced for the range $[\tau - 1.5*\text{grid-spacing}, g_{\text{max}}$.
The Lipschitz constant is set to $0.005$. \\
\textbf{Kernels}:
For the \texttt{single\_four\_branch}, the RBF kernel is used for both input and output kernels with the following length scales: $[[0.201, 0.198], 10.0]$, respectively.
For the \texttt{compositional\_four\_branch}, the RBF kernel is used for all channels with length scales: $[[2.018, 1.983], 2.472, 2.374, 2.077, 2.077, 10.0]$.

\section{ILLUSTRATIONS OF THE METHODS \& FURTHER EXPERIMENTS (SEC.\ \ref{sec:experiments})}\label{appendix:illustrations}

\begin{figure}[t]
	
	\begin{center}
		\subfigure[]{\label{fig_withUW:b}\includegraphics[width=0.235\textwidth]{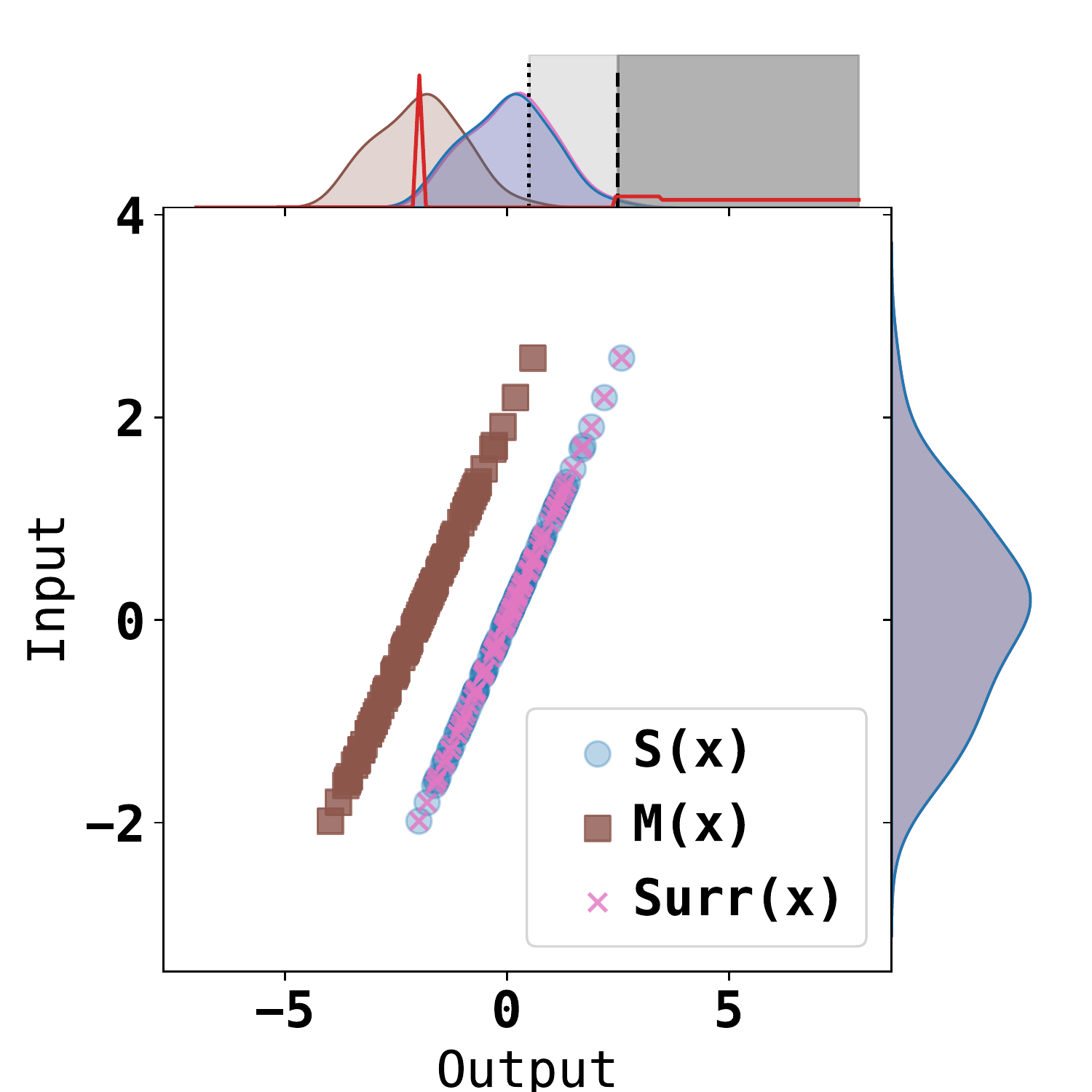}}
		\subfigure[]{\label{fig_withUW:a}\includegraphics[width=0.235\textwidth]{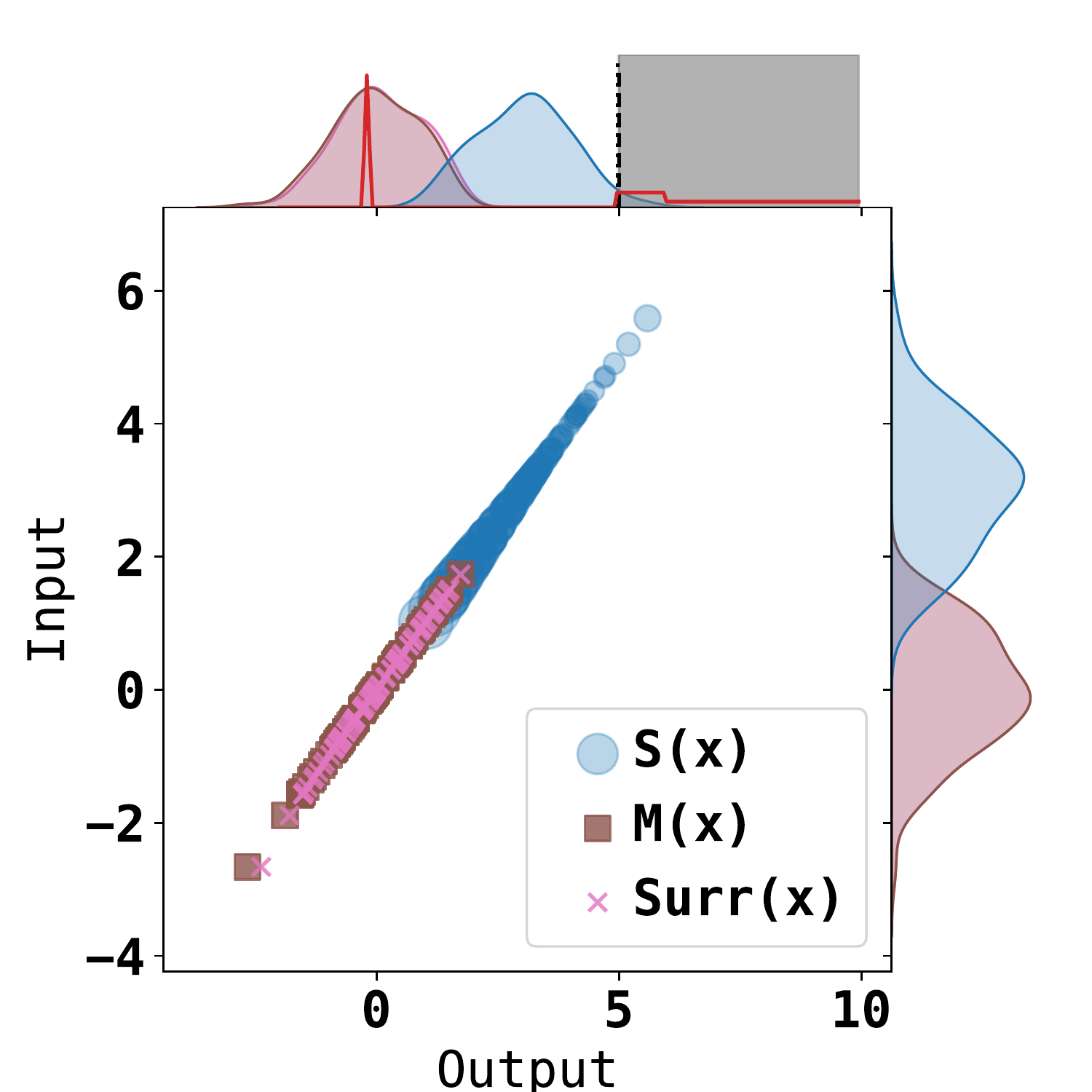}}
	\end{center}
	
	\caption{Illustration of \texttt{DPBound} (see also Fig.\ \ref{fig:linear_use_case_illustration} in the main text) and \texttt{SurrModel} for a linear mapping between Gaussian signals.
		{\textbf{(a)}} Model and system are different $S\neq M$, whereas the input distributions are identical $p_x=q_x$.
		{\textbf{(b)}} The model $M$ is the perfect model, i.e. $S=M$, but input distributions are different. Computed weights $\alpha_v$ (see Eq.\ (\ref{eq:p-alpha-joint}) in Sec.\ \ref{subsec:validation-method})  are indicated by the size of markers for $S(x)$ and the worst-case distributions w.r.t. the failure probability are indicated in red.
		The inputs and outputs of the surrogate model \texttt{SurrModel} are shown in pink.
		\label{fig:linear_use_case_illustration_WITH_UW}}
\end{figure}

\subsection{Surrogate Model \texttt{SurrModel} in the Toy Setting (Sec.\ \ref{sec:experiment_linear_use-case})}
\label{sec:appendix_uw_toy}

In Sec.~\ref{sec:experiment_linear_use-case}, an illustrative example was used to visualize the two configurations considered in the experimental setup.
Here, we additionally analyze the performance of the surrogate model (i.e.\ the \texttt{SurrModel} method from Sec.\ \ref{sec:uncertainty-wrapper-method}) for this single-component linear example under the two configurations.

We illustrated and discussed in Sec.~\ref{sec:experiment_linear_use-case} how \texttt{DPBound} can handle biases in the input distribution, as well as mismatches between the models $S$ and $M$.
On the other hand, the explicit uncertainty estimation with surrogate models fails to handle or detect mismatches in the input distribution, because the estimate of the output distribution arises from surrogate models (albeit learned on the validation data from $S$) run on the input distribution of $M$, thereby completely ignoring the real-world input distribution $p_x$ of $S$. 
To see this, note that the resulting output distribution (pink crosses) of the surrogate model in Fig.\ \ref{fig_withUW:b} lies on top of the system output distribution $S(\cdot)$ (which is different from $M(\cdot)$), whereas in Fig.\ \ref{fig_withUW:a} it basically coincides with the model output distribution $M(\cdot)$ (so that no difference is detected).
Consequently, surrogate models can detect differences due to modeling mismatches $M\neq S$, but not between input distributions $q_x\neq p_x$.

\begin{figure}
	\begin{center}
		\includegraphics[width=0.9\textwidth]{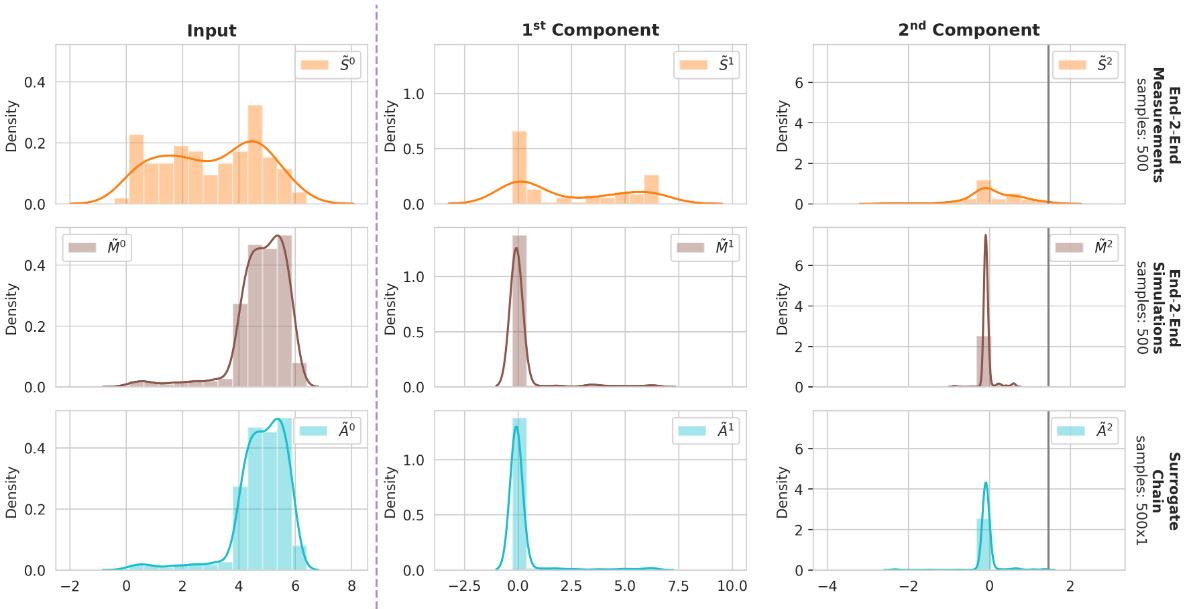}

		\caption{Figure showing the input/output signal distributions for the ``Chained Solvers'' use-case in the setting ``Biased Input--Misfit Model'' (cf.\ main Tab.\ \ref{tab:results_benchmark}; we chose this use-case for illustration purposes, as its signals are one-dimensional). \textbf{Top row:} ground-truth signals (from the system $S$). \textbf{Middle row:} simulation signals (from the model $M$). \textbf{Bottom row:} surrogate model signals (from the model $M'$ in the \texttt{SurrModel} method, see Sec.\ \ref{sec:uncertainty-wrapper-method}). \textbf{Left column:} input distributions. \textbf{Middle column:} output distributions after first component. \textbf{Right column:} final TPI distributions.\label{fig:illustration_signal_propagation}}
	\end{center}
\end{figure}

\begin{figure}
	\begin{center}
		\includegraphics[width=0.9\textwidth]{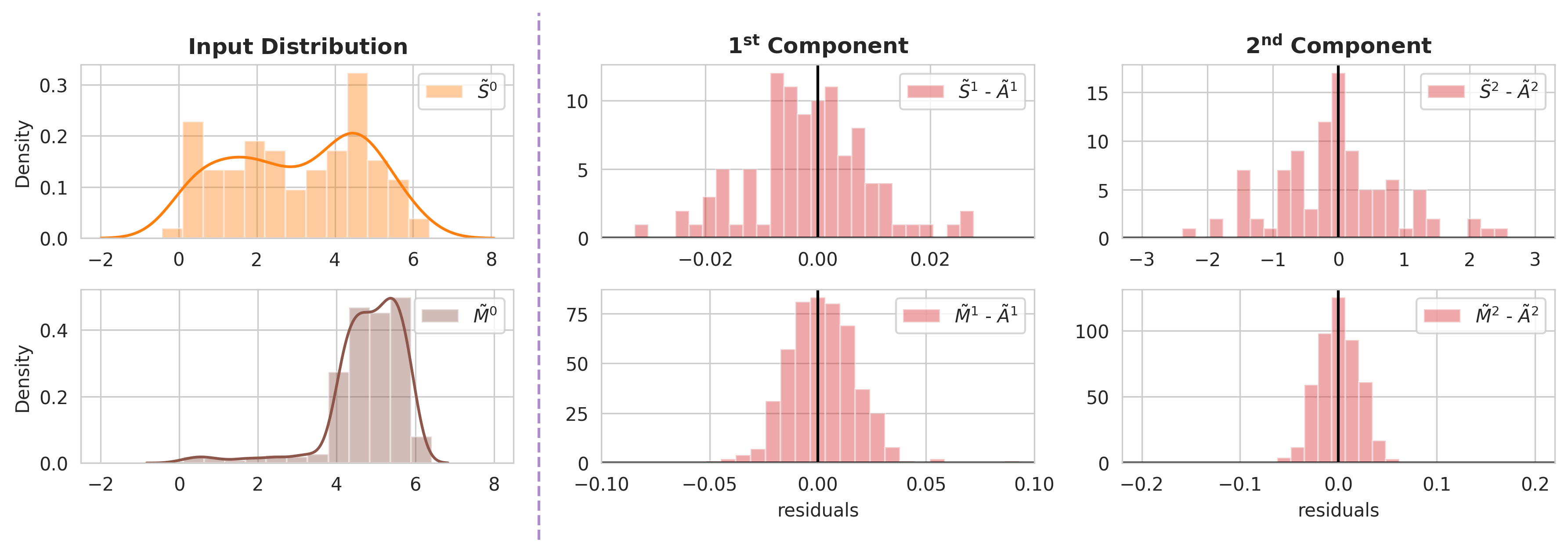}

		\caption{Figure showing the propagation of errors (residuals) of the surrogate model for the ``Chained Solvers'' use-case in the setting ``Biased Input--Misfit Model'' (cf.\ main Tab.\ \ref{tab:results_benchmark}); see Fig.\ \ref{fig:illustration_signal_propagation} for the actual signals. \textbf{Left column:} real-world input distribution (top) and simulation input distribution (bottom; note that the simulation input distribution is biased). \textbf{Top row (2nd and 3rd column):} histograms over residuals (errors) between real system and surrogate model (both starting from the real-world input distribution). \textbf{Bottom row (2nd and 3rd column):} histograms over residuals (errors) between simulation model and surrogate model (both starting from the simulation input distribution), akin to what the \texttt{SurrModel} method uses (see Eq.\ (\ref{eq:delta-surr-model}) in Sec.\ \ref{sec:uncertainty-wrapper-method}). \label{fig:illustration_error_propagation}}
	\end{center}
\end{figure}

\subsection{Illustration of Signal and Error Propagation (for the ``Chained Solvers'' Usecase, Sec.\ \ref{sec:reli-benchm-eval})}\label{sec:signal_propagation_illustration}
To illustrate the propagation of signals through the chain of subsystems, Fig.\ \ref{fig:illustration_signal_propagation} shows the propagation of signals through the system (top row), the model chain (middle row), and the surrogate model chain (bottom row, for the \texttt{SurrModel} method from Sec.\ \ref{sec:uncertainty-wrapper-method}). This is shown for the ``Chained Solvers'' use-case in the setting of ``Biased Input--Misfit Model'' (see Sec.\ \ref{sec:reli-benchm-eval} and lower right quadrant of the main Tab.\ \ref{tab:results_benchmark}); we picked this ``Chained Solvers'' use-case for visualization purposes as it has one-dimensional signals. The mismatches (errors) between those three signals are illustrated in Fig.\ \ref{fig:illustration_error_propagation}.

Fig.\ \ref{fig:illustration_error_propagation} illustrates the propagation of errors (residuals) between system/model and surrogate model through the components, in the same setting as Fig.\ \ref{fig:illustration_signal_propagation} (described in the previous paragraph). As the \texttt{SurrModel} method (Sec.\ \ref{sec:uncertainty-wrapper-method}), starts from the simulation inputs (bottom row of Fig.\ \ref{fig:illustration_error_propagation}), which may be biased w.r.t.\ the real-world input distribution (top tow), the residuals from Eq.\ (\ref{eq:delta-surr-model}) in Sec.\ \ref{sec:uncertainty-wrapper-method} used by \texttt{SurrModel} may be too small (compare lower-right vs.\ upper-right panel in Fig.\ \ref{fig:illustration_error_propagation}), finally leading to an underestimate of the failure probability by \texttt{SurrModel} in Eq.\ (\ref{eq:Fmax_estimate_UW}) (Sec.\ \ref{sec:uncertainty-wrapper-method}). This failure of the \texttt{SurrModel} method is observed in the actual experiments (main Tab.\ \ref{tab:results_benchmark} in the main text), especially for the ``biased-input'' settings.

\subsection{Dependence of \texttt{MCCP} on Its Confidence Level Parameter}\label{app:MCCP99}
In Tab.\ \ref{table:MCCP99} we corroborate our conclusions from the experiments in Sec.\ \ref{sec:reli-benchm-eval} regarding the (in)validness of the \texttt{MCCP} method.

\begin{table}[h!]
	\begin{center}
		\caption{\label{table:MCCP99}Ratio of invalid bounds (i.e.\ bounds below 1\%) produced by the \texttt{MCCP}-method run with a confidence (CL) parameter of 99\% in each of the four simulations configurations, compared (in parentheses) to the ratio for \texttt{MCCP} at 95\% CL parameter from main Tab.\ \ref{tab:results_benchmark}. 
			While the ratio of invalid bounds does decrease with the higher CL parameter of 99\%, the ratio does not decrease in a proportionate way down to one fifth of the ratio at 95\%-CL, especially not for the \emph{Biased Input} settings. 
			The invalidness ratio stays clearly above its 1\% validity promise (except in the easy case of \emph{Perfect Input--Perfect Model}). 
			This indicates that \texttt{MCCP}'s CL parameter is \emph{not} the main reason for its (high) level of invalidity. 
			The main reason is rather \texttt{MCCP}'s ignorance of the system input distribution and of the model misfits (see Sec.\ \ref{sec:reli-benchm-eval}).}

		\begin{tabular}{ |l|c|c| } 
			\hline
			\texttt{MCCP} at 99\% CL (95\% CL) & \textbf{Perfect Model} & \textbf{Misfit Model} \\
			\hline
			\textbf{Perfect Input} & 0\% (0\%) & 5\% (22.5\%) \\
			\hline
			\textbf{Biased Input} &47.5\% (67.5\%) & 42.5\% (67.5\%)\\
			\hline
		\end{tabular}
	\end{center}
\end{table}

\subsection{Tightness of the SDP Relaxation (Lemma \ref{lem:SDR-relaxation})}\label{sec:empirical-tightness}
Here we investigate experimentally how tight our convex (SDP) relaxation of the nonconvex bound optimization in Eq.\ (\ref{max-discrepancy-objective-alpha}) from Sec.\ \ref{subsec:validation-method} is (see also App.\ \ref{app:semidefinite-relaxation}). For this, we evaluate the \emph{minimum approximation ratio} $\widehat{\gamma}\in[0,1]$ of the SDP relaxation, as defined below Lemma \ref{lem:SDR-relaxation} in App.\ \ref{app:semidefinite-relaxation}, for each of the $440$ SDP optimizations required to produce our main results table (Tab.\ \ref{tab:results_benchmark} in Sec.\ \ref{sec:method}, which summarizes $8\cdot4\cdot5=160$ validation runs). Note that $\widehat{\gamma}=1$ would be proof of a perfectly \emph{tight} relaxation, while for example $\widehat{\gamma}=0.99$ guarantees that the relaxation was tight up to at most $1\%$. These tightness results are summarized in Tab.\ \ref{table:SDPtightness}.

\begin{table}[h!]
	\begin{center}
		\caption{\label{table:SDPtightness}Frequency of \emph{minimum approximation ratios} $\widehat{\gamma}$ of the SDP relaxations (defined below Lemma \ref{lem:SDR-relaxation} in App.\ \ref{app:semidefinite-relaxation}) for the $440$ SDP relaxations required to produce Tab.\ \ref{tab:results_benchmark} in Sec.\ \ref{sec:method}. Note that the true but unknown approximation ratio $\gamma$ of each SDP relaxation satisfies $\widehat{\gamma}\leq\gamma\leq1$. Thus, over all $160$ validation tasks from the main Tab.\ \ref{tab:results_benchmark}, $87.3\%$ of the $440$ required SDP bound optimizations are guaranteed to be at least $99\%$-tight.}
		
		\begin{tabular}{ |l|c|c| } 
			\hline
			minimum approximation ratio $\widehat{\gamma}$ & \# of SDP optimizations & \% of SDP optimizations \\
			\hline
			$0.99\leq{\widehat{\gamma}}\leq1.0$ & $384$ & $87.3\%$ \\
			$0.9~~\leq{\widehat{\gamma}}<0.99$ & $~~15$ & $~~3.4\%$ \\
			$0.1~~\leq{\widehat{\gamma}}<0.9$& $~~29$ & $~~6.6\%$ \\
			$0.0~~\leq{\widehat{\gamma}}<0.1$ & $~~12$ & $~~2.7\%$ \\
			\hline
			total \# of SDP optimizations & $440$ & $~100\%$ \\
			\hline
		\end{tabular}
	\end{center}
\end{table}

\subsection{Computational Cost \& Runtime}\label{sec:runtime}
The computational complexity of the method in terms of the validation data set size(s) and the number of components is discussed at the end of App.\ \ref{app:semidefinite-relaxation} (see also below Eq.\ (\ref{max-discrepancy-objective-alpha}) in the main text).

Empirically, the runtime required for the 160 validation runs of \texttt{DPBound} to produce Tab.\ \ref{tab:results_benchmark} on our desktop machine is 2 hours; this runtime is dominated by the semidefinite optimization steps (where we use Eqs.\ (\ref{eq:tightened-SDR-objective}--\ref{eq:tightened-SDR-last-tightening-constraint}) in place of Eq.\ (\ref{max-discrepancy-objective-alpha}), and use the concrete form Eqs.\ (\ref{eq:violation_objective}--\ref{eq:violation_lipschitz_constraint}) in place of Eq.\ (\ref{max-failure-objective-alpha})). On those same 160 validation problems, the \texttt{MCCP} method takes 2min (as \texttt{MCCP} must only propagate model inputs through the given model chain, with no optimizations to do), while \texttt{SurrModel} takes 25min (mainly spent on fitting the surrogate Gaussian Process models).

\end{document}